\documentclass[11pt]{article}
\usepackage{graphicx,psfrag,epsf}
\usepackage{enumerate}
\usepackage{natbib}
\usepackage{url} 
\graphicspath{{C:/Users/whsigris/Dropbox/HSLU/Projects/Generalized_GPBoost/plots_experiments/}}
\usepackage{amssymb, amsmath, amsthm,graphicx,bbm,float}
\usepackage[linesnumbered,ruled]{algorithm2e}
\usepackage{algorithmic}
\usepackage{color}
\usepackage[toc,page]{appendix}
\usepackage{caption}
\DeclareMathOperator*{\argmin}{argmin}
\DeclareMathOperator*{\argmax}{argmax}
\newtheorem{theorem}{Theorem}[section]

\newtheorem{proposition}[theorem]{Proposition}

\newcommand{\blind}{0}

\begin{document}
	
\newif\ifNotAnonymous
\NotAnonymoustrue

\newif\ifTwoColumn
\TwoColumnfalse

\def\spacingset#1{\renewcommand{\baselinestretch}%
	{#1}\small\normalsize} \spacingset{1}

\if0\blind
{
	\title{\bf Latent Gaussian Model Boosting}
	
	\author{Fabio Sigrist\thanks{Email: fabio.sigrist@hslu.ch. Address: Lucerne University of Applied Sciences and Arts, Suurstoffi 1, 6343 Rotkreuz, Switzerland.}\\
		Lucerne University of Applied Sciences and Arts}
	\maketitle
}
\fi

\bigskip

\spacingset{1} 

\bibliographystyle{abbrvnat}

\begin{abstract}
	Latent Gaussian models and boosting are widely used techniques in statistics and machine learning. Tree-boosting shows excellent prediction accuracy on many data sets, but potential drawbacks are that it assumes conditional independence of samples, produces discontinuous predictions for, e.g., spatial data, and it can have difficulty with high-cardinality categorical variables. Latent Gaussian models, such as Gaussian process and grouped random effects models, are flexible prior models which explicitly model dependence among samples and which allow for efficient learning of predictor functions and for making probabilistic predictions. However, existing latent Gaussian models usually assume either a zero or a linear prior mean function which can be an unrealistic assumption. This article introduces a novel approach that combines boosting and latent Gaussian models to remedy the above-mentioned drawbacks and to leverage the advantages of both techniques. We obtain increased prediction accuracy compared to existing approaches in both simulated and real-world data experiments.
\end{abstract}

\section{Introduction}
Boosting \citep{freund1996experiments, friedman2001greedy} is a machine learning technique that achieves state-of-the-art prediction accuracy \citep{chen2016xgboost, shwartz2021tabular}. This is reflected in statements such as ``[i]n general `boosted decision trees' is regarded as the most effective off-the-shelf non-linear learning method for a wide range of application problems" \citep{johnson2013learning}. In boosting, and in many other supervised machine learning algorithms, it is assumed that a potentially complex predictor function $F(\cdot)$ relates a set of predictor variables to a response variable, and that conditional on $F(\cdot)$ evaluated at the predictor variables, different samples are independent. Apart from this potentially unrealistic independence assumption, tree-boosting can have difficulty with high-cardinality categorical variables, and it produces discontinuous predictions. The latter is often unrealistic for spatial and spatio-temporal data.

Latent Gaussian models are a broad class of flexible prior models in which, conditional on latent Gaussian variables, a response variable is assumed to follow a known parametric distribution, and parameters of this distribution are related to the latent Gaussian variables. Two widely known types of latent Gaussian models are Gaussian process \citep{williams2006gaussian} and grouped, or clustered, random effects models \citep{pinheiro2006mixed}. Gaussian process models are used for modeling, for instance, time series, spatial, and spatio-temporal data. Further, grouped random effects models are used for modeling data with a grouping structure. In particular, grouped random effects models can be seen as an approach for modeling categorical variables with possibly high-cardinality, as every categorical variable corresponds to a grouping and vice versa. Latent Gaussian models have the advantage that they are probabilistic models which allows for making probabilistic predictions. Besides, the explicit modeling of dependence allows for efficient learning of the predictor function $F(\cdot)$. A drawback of existing latent Gaussian models is that the prior mean is often assumed to be either zero or to be a linear function of predictor variables. Both the zero-mean and the linearity assumption can be unrealistic, and higher prediction accuracy can be obtained by relaxing these assumptions; see, e.g., our experiments in Sections \ref{simul} and \ref{data_appl}.

The goal of this article is to combine boosting and latent Gaussian models for non-Gaussian data distributions. Specifically, we consider a class of models where the response variable follows a known parametric distribution, and a parameter of this distribution is related to the sum of a non-parametric function and a latent Gaussian variable. We propose to model the predictor function $F(\cdot)$ by an ensemble of base learners, such as regression trees \citep{breiman1984classification}, learned in a stage-wise manner by doing functional gradient descent steps in a boosting framework, and the hyperparameters of the covariance structure of the latent Gaussian model are jointly estimated with the predictor function; see Section \ref{LaGaBoostsec} for more details. 

Our novel approach allows for relaxing both the independence assumption in boosting and the linearity assumption in latent Gaussian models in a flexible non-parametric way. Further, it allows for obtaining continuous, or smooth, predictions for predictor variables such as spatial coordinates while at the same time being able to capture non-linearities, discontinuities, and interactions for predictor variables for which this is desirable. In addition, the use of grouped random effects is as a way for dealing with high-cardinality categorical variables in tree-boosting. As we show in our experiments in Sections \ref{simul} and \ref{data_appl}, our novel approach leads to higher prediction accuracy compared to both existing boosting algorithms and linear latent Gaussian models.

\subsection{The View of Latent Gaussian Models as Priors and Regularizers}\label{priorsreg}
An algorithm for learning a predictor function $F(\cdot)$, which relates predictor variables to a response variable, should result in an estimate $\hat F(\cdot)$ that has both low bias and low variance. Intuitively, a low bias estimator $\hat F(\cdot)$, such as a flexible machine learning model, can have high variance if the complexity of the function $F(\cdot)$ is large relative to the sample size. Examples of data for which this can occur include, first, time series, spatial, and spatio-temporal data where the amount of variation over space and/or time is large relative to the sample size and, second, data with high-cardinality categorical variables where the number of categories is large relative to the sample size. 

Modern supervised machine learning approaches such as deep neural networks and tree-boosting typically have low bias but need to apply some form of regularization to avoid high variance in $\hat F(\cdot)$. General-purpose regularization options for boosting include early stopping, learning rate shrinkage, and restrictions on the base learners such as the depth of trees and the minimal number of samples per leaf. However, as we argue in this article, for applications involving, e.g., spatial data or high-cardinality categorical variables, it can be advantageous to apply problem-specific regularization which incorporates available prior knowledge instead of relying on agnostic general-purpose regularization.

Prior models such as latent Gaussian models which explicitly model residual dependence among data can be interpreted as applying a form of regularization. For instance, an important prior assumption of Gaussian processes is that observations that are close together in space and/or time, or any other feature that defines a Gaussian process, are ``more similar to each other than distant samples". For spatial data, this prior assumption is often referred to as Tobler's first law of geography \citep{tobler1970computer}. Such a prior model implies regularization in the sense that predictions for points that are close together are similar, and that the amount of similarity varies in a continuous, or potentially smooth, manner with distance. Further, heuristically, a prior assumption of grouped random effects models is that different group effects are similar to some degree, and deviations from a global average are stochastic and identically distributed. Crucially, important characteristics of a latent Gaussian model such as the speed at which the dependency decays over space and/or time, the smoothness, the amount of variation over space and/or categories, and thus the amount of regularization implied by the prior is characterized by hyperparameters which can be learned from data. Our proposed approach allows for incorporating this reasonable prior knowledge and thus for applying explicit data-specific regularization in boosting algorithms. 

Intuitively, we conjecture that the improvement in prediction accuracy of our novel approach over classical independent tree-boosting is the larger, the more categories a categorical variable has and the faster the covariance decays over space and/or time or, in other words, the higher the complexity of $F(\cdot)$ is compared to the sample size since appropriate regularization is more important in these cases. This hypothesis is confirmed in simulated experiments in Section \ref{simul_ext}.

\subsection{Related Work}
For Gaussian data, existing approaches for combining \ifNotAnonymous Gaussian process and \fi grouped random effects models with machine learning algorithms include \citet{hajjem2011mixed,sela2012re,fu2015unbiased,hajjem2014mixed}\ifNotAnonymous, \citet{sigrist2020gaussian}\fi, \citet{Griesbach2021}, \citet{rabinowicz2021trees}, and \citet{saha2021random}.

For non-Gaussian data, there exists little prior work on combining non-linear machine learning methods with latent Gaussian models. For the special case of grouped random effects, \citet{hajjem2017generalized}, \citet{fokkema2018detecting}, and \citet{speiser2020bimm} propose algorithms that use regression trees to model the function $F(\cdot)$. \citet{speiser2019bimm} and \citet{pellagatti2020generalized} extend these algorithms by replacing trees with random forests. However, all of these methods are heuristically motivated. In particular, it is unclear which objective functions these algorithms minimize \textendash $ $ they do not maximize a marginal or approximate marginal log-likelihood neither in a component-wise way nor using an EM algorithm \textendash $ $ and whether and to which values they converge.

A straightforward alternative to the use of Gaussian processes and grouped random effects is to simply include the variables that define the latent Gaussian model, such as spatial coordinates, time points, and categorical variables, in the deterministic predictor function $F(\cdot)$ of a statistical or machine learning model. A special example of this is the approach \citet{hothorn2010model} where splines are used to model spatial effects and ridge regression is used to model grouped random effects. However, while the adoption of splines avoids discontinuities in predictions, this approach has several drawbacks compared to using latent Gaussian models. First, the hyperparameters, and thus the amount of regularization or smoothing, cannot be learned from data and need to be chosen using, e.g., cross-validation and, second, since the base learners are deterministic, probabilistic predictions cannot be made. Further, splines have the disadvantage that they suffer from the so-called ``curse of dimensionality" when the dimension of the ``locations" is large and the locations are thus sparse in space. This approach can thus not be used in situations where Gaussian processes are applied to higher-dimensional non-spatial ``locations" as is often done in machine learning applications of Gaussian processes.

The linearity assumption in mixed effects models can also be relaxed by using splines or generalized additive models \citep{Hastie1986, wood2017generalized} for modeling the predictor function $F(\cdot)$; see, e.g., \citet{tutz2007boosting} and \citet{groll2012regularization}. However, one has to assume a certain functional form with only limited possibility for interaction effects for the predictor function by specifying, for instance, main and second-order interaction effects. In general, this can thus result in model misspecification.

\section{A Non-parametric Latent Gaussian Model}\label{model_def}
We assume that the response variable $y=(y_1,\dots,y_n)^T\in \mathbb{R}^n$ follows a parametric distribution which has a density $p(y|\mu,\xi)$ with respect to a sigma finite product measure with parameters $\mu\in \mathbb{R}^n$ and $\xi\in\Xi\subset \mathbb{R}^r$. The focus of this article is on non-Gaussian densities $p(y|\mu,\xi)$. If $p(y|\mu,\xi)$ is a Gaussian density, calculations simplify as the required marginalization can be done analytically\ifNotAnonymous; see \citet{sigrist2020gaussian}\fi. Examples of $p(y|\mu,\xi)$ include Bernoulli and Poisson densities for binary classification and Poisson regression. The parameter $\mu$ is related to the sum of a predictor function $F(\cdot)$ evaluated at predictor variables and a latent Gaussian variable:
\begin{equation}\label{regre_form}
\mu=F(X)+Zb,~~~~b\sim \mathcal{N}(0,\Sigma),
\end{equation}
where $F(X)$ is the row-wise evaluation of a function $F(\cdot):\mathbb{R}^p\rightarrow\mathbb{R}$, $F(X)=(F(X_1),\dots,F(X_n))^T$, and $X_{i}=(X_{i1}\dots,X_{ip})^T\in\mathbb{R}^{p}$ is the $i$-th row of $X\in\mathbb{R}^{n\times p}$ containing predictor variables for observation $i$, $i=1,\dots,n$. For notational simplicity, we assume that the distribution $p(y|\mu,\xi)$ is parameterized in a way such that $\mu\in\mathbb{R}^n$. Otherwise, if the support of $\mu$ is not $\mathbb{R}^n$, the model needs to be reparametrized using, e.g., a so-called link function. Further, we assume that conditional on $\mu$, the data is independent:
$$p(y|\mu,\xi)=\prod_{i=1}^n p(y_i|\mu_i,\xi).$$
Any additional, auxiliary or hyper-, parameters of the likelihood $p(y|\mu,\xi)$ are denoted by $\xi$. In many situations such as classification and Poisson regression, there are no additional parameters.

We assume that $F(\cdot)$ is a function in a function space $\mathcal{H}$ that is the linear span of a set $\mathcal{S}$ of so-called base learners $f_j(\cdot):\mathbb{R}^p\rightarrow \mathbb{R}$. Classes of base learners include, e.g., linear functions \citep{buehlmann2006boosting}, smoothing splines \citep{buhlmann2003boosting}, wavelets \citep{saberian2011taylorboost}, \ifNotAnonymous reproducing kernel Hilbert space (RKHS) regression functions \citep{sigrist2019KTBoost}, \fi and regression trees \citep{breiman1984classification}, with the latter being the most popular choice. For defining functional derivatives, we additionally assume that the space $\mathcal{H}$ is normed. For instance, assuming that the $X_i$'s are identically distributed and that all $F\in\mathcal{H}$ are square integrable with respect to the law of $X_1$, a norm on $\mathcal{H}$ can defined by the inner product $\langle F,G\rangle = E_{X_1}(F(X_1)G(X_1))$ for $F,G\in \mathcal{H}$.

Examples of latent Gaussian variables  $b \in \mathbb{R}^m$ include finite-dimensional versions of Gaussian processes and/or grouped random effects. We assume that the covariance matrix $\text{Cov}(b)=\Sigma$ is parametrized by a set of parameters $\theta\in\Theta\subset \mathbb{R}^q$ whose dimensionality is often relatively low, and $\Sigma$ can depend on predictor variables $S\in \mathbb{R}^{n\times d}$. For instance, for spatial and temporal Gaussian processes, these predictor variables $S$ are locations and time points, respectively. For notational simplicity, we suppress the dependence of $\Sigma$ on its parameters $\theta$ and on $S$. Further, $Z\in\mathbb{R}^{n\times m}$ are predictor variables which relate the random variable $b$ to $\mu$. Often, $Z$ is simply an incidence matrix with entries in $\{0,1\}$. For instance, for grouped random effects, $Z$ consists of dummy variables that encode categorical variables. In general, $Z$ can also contain continuous predictor variables, e.g., in the case of random coefficient models \citep{gelfand2003spatial}. Note that, conditional on $F(X)$ and $Z$, dependence among the response variable $y$ can arise either due to the matrix $Z$ being non-diagonal or due to the covariance matrix $\Sigma$ being non-diagonal. 

In summary, we distinguish between three sets of predictor variables: $X$ with input variables for the predictor function $F(\cdot)$, $S$ which determines the covariance structure of the random variable $b$, and $Z$ which relates $b$ to $\mu$ and thus also determines the covariance structure of $\mu$ and $y$. Note that these three sets of predictor variables may or may not be over-lapping. If e.g., $X$ and $S$ contain disjoint sets of predictor variables, one assumes that there are no interactions among them. On the other hand, if, for instance, spatial locations in $S$ are also included in $X$, interactions among locations and other predictor variables in $X$ can be modeled.

In comparison to our approach, existing boosting algorithms, and many supervised machine learning algorithms in general, do not distinguish between the different types of predictor variables $X$, $S$, and $Z$, and one essentially has two options: either ignore the additional predictor variables in $S$ and $Z$ or include them in the set of predictor variables $X$ for the predictor function $F(\cdot)$. It goes without saying that the former is not a good option as potentially important information is neglected. Furthermore, the second option can result in the high variance problem mentioned in the introduction, and this translates into inferior prediction accuracy; see, e.g., our experiments in Sections \ref{simul} and \ref{data_appl}. Besides, existing boosting algorithms assume that the data $y$ is independent conditional on $F(X)$ and thus ignore any potential residual correlation. Further, in most latent Gaussian models, it is assumed that $F(\cdot)$ is either a linear function, $F(X)=X\beta$, or that $F(\cdot)$ is simply zero, $F(X)=0$. 

For notational simplicity, we assume that only one parameter $\mu$ of the data distribution $p(y|\mu,\xi)$ is related to a latent Gaussian variable. However, the extension to multivariate data and/or the situation where multiple parameters depend on potentially multiple Gaussian variables is straightforward. Also note that we assume that the latent variable $b$ follows a Gaussian distribution, but moderate violations of this assumption have been shown to have only a small effect on prediction accuracy in the context of generalized linear mixed models \citep{mcculloch2011misspecifying}.

\subsection{Definition of Learners}\label{def_learn}
The marginal density of the response $y$ is given by
\begin{equation}\label{me_dens}
p(y|F,\theta,\xi)=\int p(y|\mu,\xi)p(b|\theta)db.
\end{equation}
Ideally, we would like to minimize the empirical risk functional
\begin{equation*}
R(F(\cdot),\theta, \xi): ~~(F(\cdot), \theta, \xi) ~\mapsto~ -\log(p(y|F,\theta,\xi))\Big|_{F=F(X)}.
\end{equation*}
If $p(y|\mu,\xi)$ is a Gaussian distribution, the marginalization in \eqref{me_dens} can be done analytically. For non-Gaussian data, however, an approximation has to be used. In order that an approximation is applicable for the boosting algorithms presented in this article, it needs to fulfill two requirements. First, one must be able to compute it efficiently as this needs to be done repeatedly. Second, the gradient with respect to $F(\cdot)$ must be computable in an efficient way. 

Our goal is thus to find the joint minimizer
\begin{equation}\label{optim_def}
(\hat F(\cdot), \hat \theta, \hat \xi) =\argmin_{(F(\cdot),\theta,\xi) \in (\mathcal{H},\Theta,\Xi)}R^A(F(\cdot),\theta,\xi),
\end{equation}
where $R^A(F(\cdot),\theta, \xi)$ is an approximate empirical risk functional
\begin{equation}\label{obj_func}
R^A(F(\cdot),\theta, \xi): ~~(F(\cdot), \theta, \xi) ~\mapsto~ L^A(y|F,\theta,\xi)\Big|_{F=F(X)},
\end{equation}
and $L^A(y|F,\theta,\xi)$ is an approximation to the negative logarithmic marginal likelihood $-\log(p(y|F,\theta,\xi))$. Note that $R^A(F(\cdot),\theta, \xi)$ is calculated by evaluating $F(\cdot)$ at $X$ and then calculating $L^A(y|F=F(X),\theta,\xi).$ I.e., the risk functional $R^A(F(\cdot),\theta, \xi)$ is, in general, infinite dimensional in its first argument and finite dimensional in its other arguments.

\subsection{The Laplace Approximation}
In this article, we focus on the Laplace approximation \citep{tierney1986accurate} for approximating the marginal likelihood in \eqref{me_dens}. However, other approximations that satisfy the above-mentioned requirements can equally well be used. For instance, if $p(y|\mu,\xi)p(b|\theta)$ factors into low-dimensional components, numerical integration, such as adaptive Gauss-Hermite quadrature, can be used to approximate \eqref{me_dens}. Examples, where this applies, are single-level grouped random effects models. Another potential approximation is expectation propagation (EP) \citep{minka2001expectation}. Depending on the data distribution, for instance, for binary classification, this can lead to more accurate approximations \citep{kuss2005assessing}, but it is computationally more demanding than the Laplace approximation. 

For applying the Laplace approximation, we assume that $p(y_i|\mu_i,\xi)$ is three times differentiable in $\mu_i$. The Laplace approximation for \eqref{me_dens} is given by
\ifTwoColumn
\begin{equation}\label{laplace}
\begin{split}
p(y|F,\theta,\xi) \approx &  p(y|\tilde \mu,\xi) p(\tilde b|\theta)\\
&\cdot \text{det}\left(Z^T\tilde WZ+\Sigma^{-1}\right)^{- 1/2} (2\pi)^{m/2},
\end{split}
\end{equation}
\else
\begin{equation}\label{laplace}
p(y|F,\theta,\xi)\approx  p(y|\tilde \mu,\xi) p(\tilde b|\theta) \text{det}\left(Z^T\tilde WZ+\Sigma^{-1}\right)^{- 1/2}(2\pi)^{m/2},
\end{equation}
\fi
where $\tilde b$ is the mode of $p(y|b,F,\xi)p(b|\theta)$,
\begin{equation*}
\begin{split}
\tilde b &=\argmax_{b}{p(y|\mu,\xi)p(b|\theta)} \\
&=\argmax_{b}{\log p(y|\mu,\xi) - \frac{1}{2} b^T\Sigma^{-1} b},
\end{split}
\end{equation*}
$\tilde \mu = F(X)+Z\tilde b$, and $\tilde W\in\mathbb{R}^{n\times n}$ is a diagonal matrix with entries
$$(\tilde W)_{ii}=-\frac{\partial^2 \log p(y_i| \mu_i,\xi)}{\partial \mu_i^2}\Bigg|_{\mu=\tilde \mu}.$$
Note that $\tilde b$ depends on $F=F(X)$, $\theta$, and $\xi$, but we suppress this dependence for notational simplicity. The mode can be found, for instance, using Newton's method.

Modulo constant terms that do not depend on $\theta$, $\xi$, or $F$, the Laplace approximation to the negative log-marginal likelihood $-\log(p(y|F,\theta,\xi))$ is given by
\ifTwoColumn
\begin{equation}\label{LA_loss}
\begin{split}
L^{LA}(y,F,\theta,\xi)=&-\log p(y|\tilde \mu,\xi) + \frac{1}{2} \tilde b^T\Sigma^{-1} \tilde b\\ 
&+  \frac{1}{2}\log\text{det}\left(\Sigma Z^T\tilde WZ+I_m\right).
\end{split}
\end{equation}
\else
\begin{equation}\label{LA_loss}
L^{LA}(y,F,\theta,\xi)=-\log p(y|\tilde \mu,\xi) + \frac{1}{2} \tilde b^T\Sigma^{-1} \tilde b +  \frac{1}{2}\log\text{det}\left(\Sigma Z^T\tilde WZ+I_m\right).
\end{equation}
\fi
Since
$$ p(y|F,\theta,\xi) = \frac{p(y|F,b,\xi)p(b|\theta)}{p(b|y,\theta,\xi)},$$
the Laplace approximation in \eqref{laplace} is equivalent to the following Gaussian approximation to the posterior $p(b|y,\theta,\xi)$:
\begin{equation}\label{LA_post}
p(b|y,\theta,\xi) \approx \mathcal{N}\left(\tilde b,\left(Z^T\tilde WZ+\Sigma^{-1}\right)^{-1}\right).
\end{equation}

\subsubsection{Gradients}
For the boosting algorithms introduced in the following, we need to calculate $\frac{\partial L^{LA}(y,F,\theta,\xi)}{\partial F}$ and also $\frac{\partial L^{LA}(y,F,\theta,\xi)}{\partial \theta}$ and $\frac{\partial L^{LA}(y,F,\theta,\xi)}{\partial \xi}$ if, e.g., a first-order optimization method is used for minimizing with respect to $\theta$ and $\xi$. These are obtained as follows.

\begin{proposition}\label{prop_grad_LA}
	The gradients with respect to $F$, $\theta$, and $\xi$ of the approximate negative logarithmic marginal likelihood of the Laplace approximation $L^{LA}(y,F,\theta,\xi)$ in \eqref{LA_loss} are given by
	\ifTwoColumn
	\begin{equation}\label{grad_F_LA}
	\begin{split}
	\frac{\partial L^{LA}(y,F,\theta,\xi)}{\partial F_i}=&-\frac{\partial \log p(y_i|\tilde \mu_i,\xi)}{\partial \tilde\mu_i}\\ 
	& + \frac{1}{2}\text{tr}\left(\left(Z^T\tilde WZ+\Sigma^{-1}\right)^{-1} Z^T \frac{\partial \tilde W  }{\partial F_i} Z\right)\\
	& + \left(\frac{\partial L^{LA}(y,F,\theta,\xi)}{\partial \tilde b}\right)^T\frac{\partial\tilde  b}{\partial F_i},
	\end{split}
	\end{equation}
	\begin{equation}\label{grad_par_LA}
	\begin{split}
	\frac{\partial L^{LA}(y,F,\theta,\xi)}{\partial \theta_k}=&- \frac{1}{2}\tilde b^T \Sigma^{-1} \frac{\partial \Sigma }{\partial  \theta_k} \Sigma^{-1} \tilde b \\
	&+ \frac{1}{2}\text{tr}\left(\left(\Sigma +(Z^T\tilde WZ)^{-1}\right)^{-1}\frac{\partial \Sigma }{\partial \theta_k} \right) \\
	&+ \left(\frac{\partial L^{LA}(y,F,\theta,\xi)}{\partial \tilde b}\right)^T\frac{\partial\tilde  b}{\partial \theta_k},
	\end{split}
	\end{equation}
	\begin{equation}\label{grad_par2_LA}
	\begin{split}
	\frac{\partial L^{LA}(y,F,\theta,\xi)}{\partial \xi_l}=&-\frac{\partial \log p(y|\tilde \mu,\xi)}{\partial \xi_l}\\ 
	&+ \frac{1}{2}\text{tr}\left(\left( Z^T\tilde WZ+\Sigma^{-1}\right)^{-1} Z^T \frac{\partial \tilde W  }{\partial \xi_l} Z\right) \\
	&+ \left(\frac{\partial L^{LA}(y,F,\theta,\xi)}{\partial \tilde b}\right)^T\frac{\partial\tilde  b}{\partial \xi_l},
	\end{split}
	\end{equation}
	for $i=1,\dots,n,$ $k=1,\dots,q,$ $l=1,\dots,r,$ 
	\else
	\begin{equation}\label{grad_F_LA}
	\begin{split}
	\frac{\partial L^{LA}(y,F,\theta,\xi)}{\partial F_i}=&-\frac{\partial \log p(y_i|\tilde \mu_i,\xi)}{\partial \tilde\mu_i} + \frac{1}{2}\text{tr}\left(\left( Z^T\tilde WZ+\Sigma^{-1}\right)^{-1}\left( Z^T \frac{\partial \tilde W  }{\partial F_i} Z\right)\right)\\
	& + \left(\frac{\partial L^{LA}(y,F,\theta,\xi)}{\partial \tilde b}\right)^T\frac{\partial\tilde  b}{\partial F_i},~~~~ i=1,\dots,n,
	\end{split}
	\end{equation}
	\begin{equation}\label{grad_par_LA}
	\begin{split}
	\frac{\partial L^{LA}(y,F,\theta,\xi)}{\partial \theta_k}=&- \frac{1}{2}\tilde b^T \Sigma^{-1} \frac{\partial \Sigma }{\partial  \theta_k} \Sigma^{-1} \tilde b + \frac{1}{2}\text{tr}\left(\left(\Sigma +(Z^T\tilde WZ)^{-1}\right)^{-1}\frac{\partial \Sigma }{\partial \theta_k} \right) \\
	&+ \left(\frac{\partial L^{LA}(y,F,\theta,\xi)}{\partial \tilde b}\right)^T\frac{\partial\tilde  b}{\partial \theta_k}, ~~~~ k=1,\dots,q,
	\end{split}
	\end{equation}
	\begin{equation}\label{grad_par2_LA}
	\begin{split}
	\frac{\partial L^{LA}(y,F,\theta,\xi)}{\partial \xi_l}=&-\frac{\partial \log p(y|\tilde \mu,\xi)}{\partial \xi_l}  + \frac{1}{2}\text{tr}\left(\left( Z^T\tilde WZ+\Sigma^{-1}\right)^{-1}\left( Z^T \frac{\partial \tilde W  }{\partial \xi_l} Z\right)\right) \\
	&+ \left(\frac{\partial L^{LA}(y,F,\theta,\xi)}{\partial \tilde b}\right)^T\frac{\partial\tilde  b}{\partial \xi_l}, ~~~~ l=1,\dots,r,
	\end{split}
	\end{equation}
	\fi
	where
	\begin{equation}\label{grad_loss_mode}
	\frac{\partial L^{LA}(y,F,\theta,\xi)}{\partial \tilde b_j}
	=\frac{1}{2}\text{tr}\left(\left(Z^T\tilde WZ+\Sigma^{-1}\right)^{-1} Z^T\frac{\partial \tilde W}{\partial \tilde b_j}Z\right),
	\end{equation} 
	\begin{equation}\label{grad_mode}
	\frac{\partial \tilde b}{\partial F_i} =- \left(Z^T\tilde W Z + \Sigma^{-1}\right)^{-1}Z^T\tilde W_{\cdot i}
	\end{equation}
	\begin{equation}\label{grad_mode_par}
	\frac{\partial \tilde b}{\partial \theta_k} =\left( Z^T\tilde W Z + \Sigma^{-1}\right)^{-1}  \Sigma^{-1} \frac{\partial\Sigma}{\partial \theta_k} Z^T\frac{\partial \log p(y|\tilde \mu,\xi)}{\partial  \tilde \mu} ,
	\end{equation}
	\begin{equation}\label{grad_mode_par2}
	\frac{\partial \tilde b}{\partial \xi_l} = \left( Z^T\tilde W Z + \Sigma^{-1}\right)^{-1} Z^T\frac{\partial^2 \log p(y|\tilde \mu,\xi)}{\partial  \xi_l\partial \tilde \mu},
	\end{equation}
	$\tilde W_{\cdot i}$ denotes column $i$ of $\tilde W$,
	$\frac{\partial\tilde W}{\partial F_i}=\text{diag}\left(-\frac{\partial^3 \log p(y_i|\tilde \mu_i,\xi)}{\partial\tilde \mu_i^3}\right)$, $\frac{\partial \tilde W  }{\partial \xi_l}=\text{diag}\left(-\frac{\partial^3 \log p(y_i|\tilde \mu_i,\xi)}{\partial\tilde \mu_i^2\partial \xi_l}\right)$, $\frac{\partial \tilde W}{\partial \tilde b_j} = \text{diag}\left( -\frac{\partial^3 \log p(y_i|\tilde \mu_i,\xi)}{\partial\tilde \mu_i^3}Z_{ij}\right)$.
\end{proposition}

\begin{proof}[Proof of Proposition \ref{prop_grad_LA}]
	The derivation is similar as in \citet[Chapter 5.5.1]{williams2006gaussian}. All three gradients are sums of the explicit derivatives of $L^{LA}(y,F,\theta,\xi)$ and implicit derivatives through the dependency of $\tilde b$ on $F$, $\theta$, and $\xi$. The explicit derivatives with respect to $F$, $\theta$, and $\xi$, ignoring any dependency through $\tilde b$, are given in the first two summands in \eqref{grad_F_LA}, \eqref{grad_par_LA}, and \eqref{grad_par2_LA}. For the implicit derivatives, we first note that
	$$\frac{\partial L^{LA}(y,F,\theta,\xi)}{\partial \tilde b_j}=\frac{1}{2}\text{tr}\left(\left(Z^T\tilde WZ+\Sigma^{-1}\right)^{-1} Z^T\frac{\partial \tilde W}{\partial \tilde b_j}Z\right)$$
	where we use the fact that the derivative of the first two terms in \eqref{LA_loss} with respect to $\tilde b$ vanishes, and 
	$$ \frac{\partial \tilde W}{\partial \tilde b_j} = - \text{diag}\left( \frac{\partial^3 \log p(y_i|\tilde \mu_i,\xi)}{\partial\tilde \mu_i^3}Z_{ij}\right)$$
	since 
	\begin{equation*}
	\begin{split}
	\frac{\partial \tilde W_{ii}}{\partial \tilde b_j}&=\left(\frac{\partial \tilde W_{ii}}{\partial \tilde\mu}\right)^T\frac{\partial \tilde\mu}{\partial \tilde b_j}\\
	&=\frac{\partial \tilde W_{ii}}{\partial\tilde \mu_i}Z_{ij}\\
	&=-\frac{\partial^3 \log p(y_i|\tilde \mu_i,\xi)}{\partial \mu_i^3}Z_{ij}.
	\end{split}
	\end{equation*}
	
	To find $\frac{\partial \tilde b}{\partial F_i}$, we differentiate
	\begin{equation}\label{norm_eq_mode}
	\begin{split}
	0 & = \frac{\partial }{\partial \tilde b}\left(\log p(y|\tilde \mu,\xi) - \frac{1}{2} \tilde b^T\Sigma^{-1} \tilde b\right) \\
	&=Z^T\frac{\partial \log p(y|\tilde \mu,\xi)}{\partial \tilde  \mu}- \Sigma^{-1} \tilde b
	\end{split}
	\end{equation}
	with respect to $F_i$ and obtain
	\ifTwoColumn
	\begin{equation*}
	\begin{split}
	0 =& Z^T\frac{\partial^2 \log p(y|\tilde \mu,\xi)}{\partial \tilde \mu_i\partial \tilde \mu} \\ 
	&+ \frac{\partial }{\partial \tilde b} \left(Z^T\frac{\partial \log p(y|\tilde \mu,\xi)}{\partial \tilde  \mu} - \Sigma^{-1} \tilde b\right)\frac{\partial \tilde b}{\partial F_i}\\
	=& -Z^T \tilde W_{\cdot i} +  \left(-Z^T\tilde W Z- \Sigma^{-1}\right)\frac{\partial \tilde b}{\partial F_i},
	\end{split}
	\end{equation*}
	\else
	\begin{equation*}
	\begin{split}
	0 =& Z^T\frac{\partial^2 \log p(y|\tilde \mu,\xi)}{\partial \tilde \mu_i\partial \tilde \mu}  + \frac{\partial }{\partial \tilde b} \left(Z^T\frac{\partial \log p(y|\tilde \mu,\xi)}{\partial \tilde  \mu} - \Sigma^{-1} \tilde b\right)\frac{\partial \tilde b}{\partial F_i}\\
	=& -Z^T \tilde W_{\cdot i} +  \left(-Z^T\tilde W Z- \Sigma^{-1}\right)\frac{\partial \tilde b}{\partial F_i},
	\end{split}
	\end{equation*}
	\fi
	where $\tilde W_{\cdot i}=\frac{\partial^2 \log p(y|\tilde \mu,\xi)}{\partial \tilde \mu_i\partial \tilde \mu}$ is column $i$ of $\tilde W$, i.e. a vector of $0$'s except for the $i$-th entry which is given by $\frac{\partial^2 \log p(y_i|\tilde \mu_i,\xi)}{\partial \tilde \mu_i^2}$. The statement in Equation \eqref{grad_mode} thus follows. 
	Similarly, multiplying Equation \eqref{norm_eq_mode} with $\Sigma$ and differentiating it with respect to $\theta_k$ gives
	\begin{equation*}
	0 = \frac{\partial\Sigma}{\partial \theta_k} Z^T\frac{\partial \log p(y|\tilde \mu,\xi)}{\partial  \tilde \mu}  + \left(-\Sigma Z^T\tilde W Z- I_m\right)\frac{\partial \tilde b}{\partial \theta_k},
	\end{equation*}
	from which we obtain Equation \eqref{grad_mode_par} by multiplying with $\Sigma^{-1}$. Equation \eqref{grad_mode_par2} follows analogously.
\end{proof}
We note that in our software implementation, we use different equivalent versions of the above result depending on the specific latent Gaussian model for computational efficiency and stability. If $Zb$ consists of only grouped random effects, we use the version presented in Proposition \ref{prop_grad_LA} except that in \eqref{grad_par_LA}, we replace $\left(\Sigma +(Z^T\tilde WZ)^{-1}\right)^{-1}\frac{\partial \Sigma }{\partial \theta_k}$ with the equivalent expression $\left(Z^T\tilde WZ+\Sigma^{-1}\right)^{-1}\Sigma^{-1}\frac{\partial \Sigma }{\partial \theta_k}Z^T\tilde WZ$. In this case, $Z$ and $\Sigma^{-1}$ are sparse, and the random effects dimension $m$ is smaller than the number of samples $n$. It follows that a Cholesky factor for $Z^T\tilde WZ+\Sigma^{-1}$ can be computed efficiently using sparse matrix algebra, and also the remaining calculations for obtaining the gradients in Proposition \ref{prop_grad_LA} can be done efficiently. If $Zb$ contains a finite dimensional versions of a Gaussian process, we use the Sherman-Morrison-Woodbury formula
\ifTwoColumn
$\left( Z^T\tilde W Z + \Sigma^{-1}\right)^{-1} =\Sigma - \Sigma Z^T \tilde W^{1/2} \left( I_m + \tilde W^{1/2} Z \Sigma Z^T \tilde W^{1/2} \right)^{-1}\tilde W^{1/2}Z\Sigma,$
\else
$$\left( Z^T\tilde W Z + \Sigma^{-1}\right)^{-1} =\Sigma - \Sigma Z^T \tilde W^{1/2} \left( I_m + \tilde W^{1/2} Z \Sigma Z^T \tilde W^{1/2} \right)^{-1}\tilde W^{1/2}Z\Sigma,$$
\fi
factorize the matrix $I_m + \tilde W^{1/2} Z \Sigma Z^T \tilde W^{1/2}$, and, similarly as in \citet[Chapter 5.5.1]{williams2006gaussian}, adapt all calculations in Proposition \ref{prop_grad_LA} accordingly. 

\section{Latent Gaussian Model Boosting}\label{LaGaBoostsec}
We propose to do the minimization of the risk functional in \eqref{optim_def} using a novel boosting algorithm presented in the following. For known and fixed $\theta$ and $\xi$, boosting finds a minimizer of the approximate empirical risk functional $R^A(F(\cdot),\theta, \xi)$ in a greedy way by sequentially adding an update $f_m(\cdot)$ to the current estimate $F_{m-1}(\cdot)$:
\begin{equation}\label{boostupdate}
F_m(\cdot)= F_{m-1}(\cdot)+ f_m(\cdot),~~f_m\in \mathcal{S},
\end{equation}
where $f_m(\cdot)$, $m=1,\dots,M$, is chosen such that its addition results in the minimization of the risk. This minimization cannot be done analytically and an approximation is thus used. In general, such an approximation consists of either a penalized functional first-order or a functional second-order Taylor expansion of the risk around the current estimate $F_{m-1}(\cdot)$. This corresponds to functional gradient descent or functional Newton steps. \ifNotAnonymous See \citet{sigrist2018gradient} for more information on the distinction between gradient and Newton boosting. \fi

In our case, we use functional gradient descent. Specifically, $f_m(\cdot)$ is given by the least squares approximation to the vector obtained when evaluating the negative functional gradient of $R^A(F(\cdot),\theta, \xi)$ at $(F_{m-1}(\cdot), I_{X_i}(\cdot))$, $i=1,\dots,n$, where $I_{X_i}(\cdot)$ are indicator functions which equal $1$ at $X_i$ and $0$ otherwise. Equivalently, $f_m(\cdot)$ is the  minimizer of a first-order functional Taylor approximation of $R^A(F(\cdot),\theta, \xi)$ around $F_{m-1}(\cdot)$ with an $L^2$ penalty on $f(\cdot)$ evaluated at $(X_i)$\ifNotAnonymous ; see, e.g., \citet{sigrist2018gradient} for more information. \else . \fi It is easily seen that the negative G\^ateaux derivative of $R^A(F(\cdot),\theta, \xi)$ evaluated at $(F_{m-1}(\cdot), I_{X_i}(\cdot))$ is given by the vector $-\frac{\partial L^{A}(y,F,\theta,\xi)}{\partial F}\Big|_{F=F_{m-1}(X)}$ which we denote shortly as $-\frac{\partial L^{A}(y,F_{m-1},\theta,\xi)}{\partial F}$. This means that $f_m(\cdot)$ can be found as the following least squares approximation:
\begin{equation}\label{grad_boost}
f_m(\cdot)=\argmin_{f(\cdot)\in \mathcal{S}} \left\|-\frac{\partial L^{A}(y,F_{m-1},\theta,\xi)}{\partial F}-f\right\|^2,
\end{equation}
where $f=(f(X_1),\dots,f(X_n))^T$. Note that $\frac{\partial L^{A}(y,F_{m-1},\theta,\xi)}{\partial F}$ depends on the approximation used for the marginal log-likelihood. For the Laplace approximation, this is given in Proposition \ref{prop_grad_LA}.

It has been empirically observed that damping the update in \eqref{boostupdate},
\begin{equation*}
F_m(\cdot)= F_{m-1}(\cdot)+ \nu f_m(\cdot),~~\nu>0,
\end{equation*}
results in higher prediction accuracy \citep{friedman2001greedy}. Further, functional gradient descent can also be accelerated using momentum. For instance, \citet{biau2019accelerated} and \citet{lu2019accelerating} propose to use Nesterov acceleration \citep{nesterov2004introductory} for gradient boosting.

To jointly learn $F(\cdot)$ and $(\theta, \xi)$, we propose to combine functional boosting updates in the direction of $F(\cdot)$ with coordinate descent steps in $\theta$ and $\xi$. \ifNotAnonymous The reasons for this choice are outlined in \citet{sigrist2020gaussian}. \fi The LaGaBoost Algorithm \ref{LaGaBoost} summarizes our approach. Note that, despite not being explicitly stated in Algorithm \ref{LaGaBoost}, the approximation for the negative logarithmic marginal likelihood needs to be calculated repeatedly in the algorithm whenever $L^{A}(y,F,\theta,\xi)$ is evaluated or a gradient of it is calculated.

\begin{algorithm}[ht!]
	\SetKwInOut{Input}{Input}
	\SetKwInOut{Output}{Output}
	\Input{Initial values $\theta_0\in\Theta$ and, if applicable, $\xi_0\in\Xi$, learning rate $\nu>0$, number of boosting iterations $M\in\mathbb{N}$, approximation $L^{A}(y,F,\theta,\xi)$} 
	\Output{Function $\hat F(\cdot) = F_{M}(\cdot)$, hyperparameters $\hat \theta = \theta_M$, and auxiliary parameters $\hat \xi=\xi_M$}
	\caption{LaGaBoost: Latent Gaussian model Boosting}\label{LaGaBoost}
	\begin{algorithmic}[1]
		\STATE Initialize $F_0(\cdot)=\argmin_{c\in\mathbb{R}}L^A(y,c\cdot 1,\theta_0,\xi_0)$
		\FOR{$m=1$ {\bfseries to} $M$}
		\STATE Find $(\theta_m,\xi_m)=\underset{(\theta,\xi)\in(\Theta,\Xi)}{\argmin} L^{A}(y,F_{m-1},\theta,\xi)$ using a method for convex optimization initialized with $(\theta_{m-1},\xi_{m-1})$
		\STATE Find $f_m(\cdot)=\underset{f(\cdot)\in \mathcal{S}}{\argmin}\left\|-\frac{\partial L^{A}(y,F_{m-1},\theta_m,\xi_m)}{\partial F}-f\right\|^2$ 
		\STATE Update $F_m(\cdot)= F_{m-1}(\cdot)+ \nu f_m(\cdot)$
		\ENDFOR
	\end{algorithmic}
\end{algorithm}

If the risk functional $R^A(F(\cdot),\theta, \xi)$ is convex in its arguments and $\Theta$ and $\Xi$ are convex sets, then \eqref{optim_def} is a convex optimization problem since $\mathcal{H} = span(\mathcal{S})$ is also convex. This means that there exists a unique minimizer and the LaGaBoost algorithm converges to the minimum, as long as the learning rate $\nu$ is not too large to avoid overshooting, i.e., that the risk increases when doing too large steps. Further, the computational complexity of the algorithm depends on the specific latent Gaussian variable model and the marginal likelihood approximation used. For instance, for the Laplace approximation, the calculation of Cholesky factors is usually the bottleneck. 

\subsection{Out-of-sample Learning for Hyperparameters}\label{lagaboostoost}
It has recently been observed that state-of-the-art machine learning techniques such as neural networks, kernel machines, or boosting can achieve a zero training loss and interpolate the training data while at the same time having excellent generalization properties \citep{zhang2016understanding, wyner2017explaining, belkin18a, belkin2019reconciling, bartlett2020benign}. Such an interpolation of the training data could be problematic for the hyperparameter estimation in the LaGaBoost algorithm. We propose to circumvent this potential problem by estimating the hyperparameters $\theta$ and the auxiliary parameters $\xi$ using out-of-sample validation data obtained by applying cross-validation or by partitioning the data into two disjoint training and validation sets. To avoid that the function $F(\cdot)$ and/or the parameters $\theta$ and $\xi$ are only learned on a fraction of the full data, we propose a two-step approach presented in the LaGaBoostOOS Algorithm \ref{lagaboostoos_algo}. In brief, the LaGaBoostOOS algorithm first runs the LaGaBoost algorithm on the training data and obtains predictions $\hat F_{val}$ for the function $F(\cdot)$ on the left out validation data. The parameters $\theta$ and $\xi$ are then estimated on the validation data using the predicted values $\hat F_{val}$. Finally, the LaGaBoost algorithm is run a second time on the full data while holding $\theta$ and $\xi$ fixed. When $k$-fold cross-validation is used, both the function $F(\cdot)$ and the parameters $\theta$ and $\xi$ are thus learned using the full data.
\begin{algorithm}[ht!]
	\SetKwInOut{Input}{Input}
	\SetKwInOut{Output}{Output}
	\Input{Initial values $\theta_0\in\Theta$ and, if applicable, $\xi_0\in\Xi$, learning rate $\nu>0$, number of boosting iterations $M\in\mathbb{N}$, approximation $L^{A}(y,F,\theta,\xi)$}
	\Output{Function $\hat F(\cdot) = F_{M}(\cdot)$, hyperparameters $\hat \theta = \theta_M$, and auxiliary parameters $\hat \xi=\xi_M$}
	\caption{LaGaBoostOOS: Latent Gaussian model Boosting with Out-Of-Sample hyperparameter estimation}\label{lagaboostoos_algo}
	\begin{algorithmic}[1]
		\STATE Partition the data into training and validation sets, e.g., using $k$-fold cross-validation or by partitioning the data into two disjoint sets
		\STATE Run the LaGaBoost algorithm on the training data and generate predictions $\hat F_{val}$ for the function $F(\cdot)$ on the validation data 
		\STATE Find $(\hat\theta,\hat\xi)=\argmin_{(\theta,\xi)\in(\Theta,\Xi)}L^{A}(y_{val},\hat F_{val},\theta,\xi)$ using the validation data with response variable $y_{val}$
		\STATE Run the LaGaBoost algorithm on the full data while holding the hyperparameters $\theta$ and auxiliary parameters $\xi$ fixed at $\hat \theta$ and $\hat \xi$, i.e., by skipping line 3 in Algorithm \ref{LaGaBoost}, to obtain $\hat F(\cdot)$
	\end{algorithmic}
\end{algorithm}

\subsection{Prediction}\label{prediction}
In the following, we show how predictions can be made. We distinguish between predicting observables variables $y_p$ and latent variables $\mu_p$. Let $y_p\in\mathbb{R}^{n_p}$ and $\mu_p\in\mathbb{R}^{n_p}$ denote the observable and latent random variables for which predictions should be made. The following holds true:
\begin{equation}\label{pred_obs_dis}
\begin{split}
\begin{pmatrix} b \\ \mu_p\end{pmatrix} & = \begin{pmatrix} 0 \\ F(X_p)\end{pmatrix} + 
\begin{pmatrix} (I_{m},0_{m\times m_p}) \\ Z_p \end{pmatrix}\begin{pmatrix} b \\ b_p \end{pmatrix},\\
\sim & \mathcal{N} \ifTwoColumn \small \fi \left(\begin{pmatrix} 0 \\ F(X_p)\end{pmatrix},
\begin{pmatrix} \Sigma & (\Sigma,\Sigma_{op})Z_p^T \\ Z_p(\Sigma,\Sigma_{op})^T &Z_p\begin{pmatrix} \Sigma& \Sigma_{op}\\ \Sigma_{op}^T&\Sigma_p\end{pmatrix}Z_p^T\end{pmatrix}\right)
\end{split}
\end{equation}
where $b_p\in\mathbb{R}^{m_p}$ is a latent random variable for which no corresponding data has been observed in $y$, $(I_{m},0_{m\times m_p})\in \mathbb{R}^{m\times (m + m_p)}$, $I_{m}\in \mathbb{R}^{m\times m}$ is an identity matrix,  $0_{m\times m_p}\in\mathbb{R}^{n\times m_p}$ is a matrix of zeros, the matrix $Z_p\in \mathbb{R}^{n_p\times (m + m_p)}$ relates the vector of observed and new latent variables $(b^T,b_p^T)^T\in\mathbb{R}^{m + m_p}$ to $\mu_p$, $(\Sigma,\Sigma_{op})\in\mathbb{R}^{m\times (m + m_p)}$, $\Sigma_{op} = \text{Cov}(b,b_p)$, $\Sigma_{p} = \text{Cov}(b_p)$, and $X_p\in\mathbb{R}^{n_p\times p}$ is the predictor variable matrix of the predictions. 

By the law of total probability, we have
$$p(\mu_p|y,\theta,\xi)=\int p(\mu_p|b,\theta)p(b|y,\theta,\xi)db$$
and
\begin{equation}\label{pred_obs}
p(y_p|y,\theta,\xi)=\int p(y_p|\mu_p,\xi) p(\mu_p|y,\theta,\xi)d\mu_p.
\end{equation}
If we apply the Laplace approximation, then by \eqref{LA_post}, \eqref{pred_obs_dis}, standard results for conditional distributions of multivariate Gaussian distributions, and the law of total variance, we have
$$p(\mu_p|y,\theta,\xi)\approx \mathcal{N}\left(\omega_p,\Omega_p\right),$$
where
\ifTwoColumn
\begin{equation*}
\begin{split}
\omega_p=&F(X_p)+ Z_p(\Sigma,\Sigma_{op})^T\Sigma^{-1}\tilde b,\\
=& F(X_p)+ Z_p(\Sigma,\Sigma_{op})^TZ^T\frac{\partial \log p(y|\tilde \mu,\xi)}{\partial \tilde\mu},
\end{split}
\end{equation*}
\begin{equation*}
\begin{split}
\Omega_p
=&Z_p\begin{pmatrix} \Sigma& \Sigma_{op}\\ 
\Sigma_{op}^T&\Sigma_p\end{pmatrix}Z_p^T  
\\ & - 
Z_p(\Sigma,\Sigma_{op})^T 
\left( \Sigma + (Z^T\tilde WZ)^{-1} \right)^{-1} (\Sigma,\Sigma_{op})Z_p^T,
\end{split}
\end{equation*}
\else
\begin{equation*}
\begin{split}
\omega_p=&F(X_p)+ Z_p(\Sigma,\Sigma_{op})^T\Sigma^{-1}\tilde b,\\
=& F(X_p)+ Z_p(\Sigma,\Sigma_{op})^TZ^T\frac{\partial \log p(y|\tilde \mu,\xi)}{\partial \tilde\mu},\\
\Omega_p=&Z_p\begin{pmatrix} \Sigma& \Sigma_{op}\\
\Sigma_{op}^T&\Sigma_p\end{pmatrix}Z_p^T 
- 
Z_p(\Sigma,\Sigma_{op})^T 
\left(\Sigma^{-1} - \Sigma^{-1}\left(Z^T\tilde WZ+\Sigma^{-1}\right)^{-1}\Sigma^{-1} \right) (\Sigma,\Sigma_{op})Z_p^T,\\
=&Z_p\begin{pmatrix} \Sigma& \Sigma_{op}\\ 
\Sigma_{op}^T&\Sigma_p\end{pmatrix}Z_p^T 
- 
Z_p(\Sigma,\Sigma_{op})^T 
\left( \Sigma + (Z^T\tilde WZ)^{-1} \right)^{-1} (\Sigma,\Sigma_{op})Z_p^T,
\end{split}
\end{equation*}
\fi
where in the last line, we have used the Sherman-Morrison-Woodbury formula.

Further, the integral in \eqref{pred_obs} is analytically tractable for a Bernoulli likelihood with a probit link \citep[see, e.g.,][Chapter 3.9]{williams2006gaussian}, but for other likelihoods, it needs to be numerically approximated. In our software implementation and the experiments below, we use adaptive Gauss-Hermite quadrature \citep{liu1994note} as numeric integration technique.

\subsection{Software Implementation}\label{software}
The LaGaBoost and LaGaBoostOOS algorithms based on the Laplace approximation are implemented in the \texttt{GPBoost} library written in C++ with corresponding Python and R packages; see \url{https://github.com/fabsig/GPBoost} for more information. For linear algebra calculations, we rely on the \texttt{Eigen} library \citep{eigenweb}. Sparse matrix algebra is used, in particular for calculating Cholesky decompositions, whenever covariance matrices are sparse, e.g., in the case of grouped random effects. Further, multi-processor parallelization is done using \texttt{OpenMP}. For the tree-boosting part, in particular the tree growing algorithm, we use the \texttt{LightGBM} library \citep{ke2017lightgbm}. The \texttt{GPBoost} library allows for modeling Gaussian processes, grouped random effects including nested and crossed ones, random coefficients, and combinations of the former. Further, the \texttt{GPBoost} library currently implements gradient descent with optional Nesterov acceleration and the Nelder-Mead method for minimizing with respect to the parameters $\theta$ and $\xi$ in line 3 of the LaGaBoost Algorithm \ref{LaGaBoost}.

\section{Simulated Experiments}\label{simul}
In the following, we perform simulated experiments to compare the novel LaGaBoost algorithm to alternative approaches. We simulate binary classification data from a latent Gaussian model as in \eqref{regre_form} assuming a Bernoulli likelihood with a probit link function: $y_i\in\{0,1\}$, $P(y_i=1)=\Phi(\mu_i)$, $i=1,\dots,n$, where $\Phi(\cdot)$ denotes the standard normal cumulative distribution function. For the latent Gaussian variable $Zb$, we consider both grouped random effects with a single grouping level and a spatial Gaussian process model with an exponential covariance function $c(s,s')=\sigma_1^2 \exp(-\|s-s'\|/\rho)$ where the locations $s$ are in $[0,1]^2$ and $\rho=0.1$. The marginal variance in both models is set to $\sigma^2=1$. Concerning the function $F(\cdot)$ and the predictor variables $X$, we sample independently from
\begin{equation}\label{simF}
\begin{split}
F(x)&=C_1 + C_2\cdot(2x_1+x_2^2+4\cdot 1_{\{x_3>0\}}+2\log(|x_1|)x_3),\\
& ~~~~~~~~ x=(x_1,\dots,x_9)^T,~~x\sim \mathcal{N} ( 0 , I_9).
\end{split}
\end{equation}
This function has been used previously in \citet{hajjem2014mixed} and \citet{sigrist2020gaussian} to compare non-parametric mixed effects models for Gaussian data. The constant $C_1$ is chosen such that the mean of $F(x)$ is approximately $0$, and $C_2$ is chosen such that the variance of $F(x)$ equals approximately $1$, i.e., that $F(x)$ has the same signal strength as the latent Gaussian variable.

We simulate $100$ times training data sets of size $n$ and two test data sets each also of size $n$. All models are trained on the training data and evaluated on the test data. We use a sample size of $n=5000$ for the grouped random effects with $m=500$ different groups. This corresponds to a categorical variable with $500$ different categories and $10$ samples per category. For the Gaussian process model, we use a sample size of $n=500$. The reason for using a smaller sample size is that this allows us to avoid any additional approximation error due to a large data approximation. In every simulation run, two test data sets, denoted as ``interpolation" and ``extrapolation" test sets, are generated as follows. For the grouped random effects model, the interpolation test data set consists of $n$ samples from the same $m$ groups as in the training data, and the extrapolation test data consists of $n$ samples for $m$ new groups that have not been observed in the training data. For the Gaussian process model, training data locations are samples uniformly from $[0,1]^2$ excluding $[0.5,1]^2$, the interpolation test data sets are obtained by also simulating locations uniformly in the same area, and the extrapolation test data contains locations sampled uniformly from the excluded square $[0.5,1]^2$. Figure \ref{Train_test_locs} illustrates this.

\begin{figure}[ht!]
	\ifTwoColumn
	\begin{center}
		\includegraphics[width=0.5\textwidth]{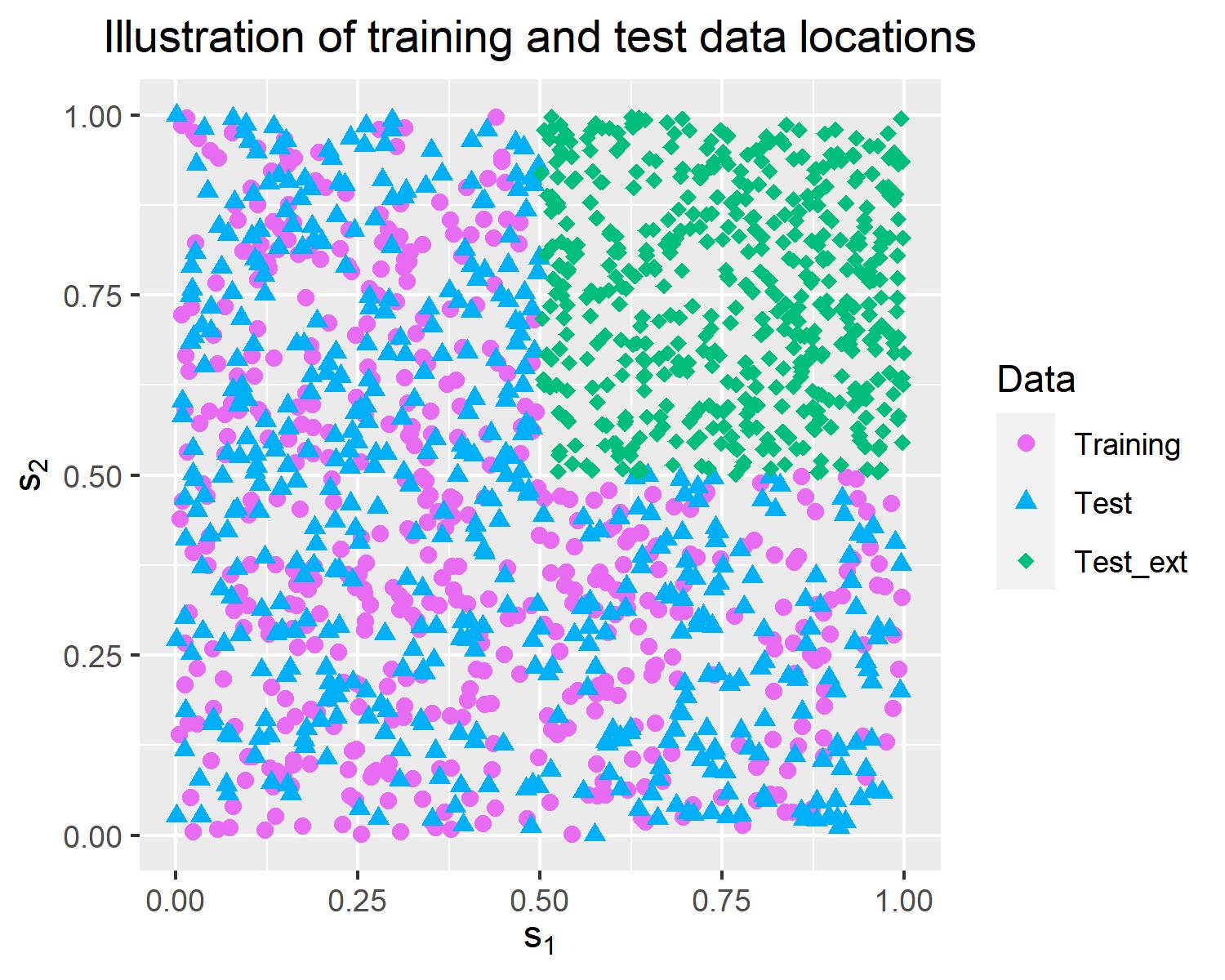}
		\caption{Example of locations for training and test data for the spatial data. `Test' and `Test\_ext' refers to locations of the interpolation and extrapolation test data sets, respectively.}
		\label{Train_test_locs}
	\end{center}
	\else
	\centering
	\includegraphics[width=0.7\textwidth]{Train_test_locs_no_sum.jpeg}
	\caption{Example of locations for training and test data for the spatial data. `Test' and `Test\_ext' refers to locations of the interpolation and extrapolation test data sets, respectively.}
	\label{Train_test_locs}
	\fi
\end{figure}

We compare the LaGaBoost algorithm based on the Laplace approximation to the following alternative approaches: linear Gaussian process and grouped random effects models for binary data with a probit link function and $F(x)=x^T\beta, \beta\in\mathbb{R}^p$, independent Newton boosting for binary data with the log loss (`LogitBoost') \citep{friedman2000additive}, and model-based gradient boosting (`mboost') \citep{hothorn2010model} with the log loss, i.e., a negative Bernoulli log-likelihood, and a probit link function. For the LogitBoost algorithm, we include the locations for the spatial data and the categorical grouping variable as additional predictor variables in the function $F(\cdot)$. For all boosting algorithms, we use trees as base learners, except for the grouped and spatial random effects in mboost. \ifNotAnonymous Learning and prediction with the LaGaBoost and LaGaBoostOOS algorithms, the linear latent Gaussian models, and LogitBoost is done using the \texttt{GPBoost} library version 0.7.0 compiled with the MSVC compiler version 19.24.28315.0 and OpenMP version 2.0. \else We use the \texttt{LightGBM} library version 2.3.0 for LogitBoost, and linear latent Gaussian models are estimated using the same C++ code as used for the LaGaBoost algorithm. \fi For the linear latent Gaussian models, the LaGaBoost algorithm, and the LaGaBoostOOS algorithm, optima for hyperparameters $\theta$ are found using Nesterov accelerated gradient descent. Further, for the linear models, the coefficients $\beta$ are also learned using Nesterov accelerated gradient descent. Note that the gradient of $L^{A}(y,F,\theta,\xi)$ with respect to $\beta$ is given by
$$\frac{\partial L^{A}(y,F,\theta,\xi)}{\partial \beta} = X^T\frac{\partial L^{A}(y,F,\theta,\xi)}{\partial F}.$$
For LogitBoost applied to the grouped random effects data, we consider the grouping variable as a numeric variable and not as a categorical variable as suggested by the authors of \texttt{LightGBM}\footnote{\url{https://lightgbm.readthedocs.io/en/latest/Advanced-Topics.html#categorical-feature-support} (retrieved on May 11, 2021)} since the number of categories is large. Concerning the mboost algorithm, we use the \texttt{mboost} R package \citep{mboost14} version 2.9-2, where spatial effects are modeled using bivariate P-spline base learner (\texttt{bspatial} with \texttt{df=6}), grouped random effects are modeled using random effects base learners (\texttt{brandom} with \texttt{df=4}), and all other predictor variables are modeled using trees as base learners. All calculations are done on a laptop with a $2.9$ GHz quad-core processor and $16$ GB of random-access memory (RAM).

Tuning parameters are chosen by simulating $10$ additional training and test sets and choosing the parameter combinations that minimize the average log loss on the test sets. In doing so, we use the union of both the interpolation and extrapolation test data sets to calculate test losses. For all boosting algorithms, we consider the following grid of tuning parameters: the number of boosting iterations $M\in \{1,\dots,1000\}$, the learning rate $\nu \in \{0.1,0.05,0.01\}$, the maximal depth of the trees $\in \{1,2,5,10\}$, and the minimal number of samples per leaf $\in \{1,10,100\}$. 

The results for the grouped and spatial data are reported in Tables \ref{results_hajjem_one_way} and \ref{results_hajjem_spatial}. We report average test error rates (`Error') and test log losses (`NegLL') for both the interpolation and extrapolation (`\_ext') test sets. Further, we calculate p-values of paired t-tests comparing the LaGaBoost algorithm to the other approaches. We find that the LaGaBoost algorithm significantly outperforms all alternative approaches in all prediction accuracy measures for both the grouped and spatial data. In Tables \ref{results_hajjem_one_way} and \ref{results_hajjem_spatial}, we additionally report the results for the LaGaBoostOOS algorithm, root mean square errors (RMSEs) and biases for the hyperparameters, and wall-clock time. Overall, we observe no large differences between the LaGaBoost and the LaGaBoostOOS algorithms. However, for the spatial data, the hyperparameter estimates of the LaGaBoostOOS algorithm have smaller RMSEs and biases compared to the LaGaBoost algorithm in line with our arguments laid out in Section \ref{lagaboostoost}. As expected, the LaGaBoostOOS algorithm has a higher computational time.

\ifTwoColumn
\begin{table}[ht!]
\centering
\caption{Results for the grouped / high-cardinality categorical variable data and a binary Bernoulli likelihood. The results for the ``extrapolation" test data sets are denoted by `\_ext'. `LaGaBst', `LogitBst', and `LGBOOS' denote
                              the LaGaBoost, LogitBoost, and LaGaBoostOOS algorithms. 
                              `NegLL' denotes the negative log-likelihood (=log loss for binary data).} 
\label{results_hajjem_one_way}
\begingroup\footnotesize
\scalebox{0.9}{
\begin{tabular}{rllll|l}
  \hline
\hline
  & LaGaBst & LinearME & LogitBst & mboost & LGBOOS \\ 
  \hline
Error & \bf{0.2373} & 0.287 & 0.3393 & 0.2825 & 0.2373 \\ 
  (sd) & (0.00674) & (0.00885) & (0.00871) & (0.00958) & (0.00677) \\ 
  \lbrack p-val\rbrack &  & [1e-86] & [2.8e-106] & [6.4e-73] & [0.66] \\ 
  NegLL & \bf{2421} & 2785 & 3030 & 2917 & 2421 \\ 
  (sd) & (45.5) & (46.1) & (28.8) & (19.7) & (45.6) \\ 
  \lbrack p-val\rbrack &  & [3.1e-105] & [8.2e-114] & [6.5e-111] & [0.67] \\ 
   \hline
Error\_ext & \bf{0.3432} & 0.4206 & 0.3533 &  & 0.3434 \\ 
  (sd) & (0.00793) & (0.00898) & (0.015) &  & (0.00783) \\ 
  \lbrack p-val\rbrack &  & [1.9e-95] & [9.2e-12] &  & [0.31] \\ 
  NegLL\_ext & \bf{3028} & 3295 & 3078 &  & 3029 \\ 
  (sd) & (31.4) & (17.9) & (55) &  & (31.7) \\ 
  \lbrack p-val\rbrack &  & [1.2e-105] & [7.5e-18] &  & [0.0036] \\ 
   \hline
RMSE $\sigma^2$ & 0.2099 & 0.3589 &  &  & 0.2141 \\ 
  Bias $\sigma^2$ & -0.1953 & -0.3536 &  &  & -0.1941 \\ 
   \hline
Time (s) & 0.6646 & 0.03906 & 0.07785 & 16.96 & 3.276 \\ 
   \hline
\hline
\end{tabular}
}
\endgroup
\end{table}

\begin{table}[ht!]
\centering
\caption{Results for the spatial data and a binary Bernoulli likelihood. See the caption of Table \ref{results_hajjem_one_way} 
                                for information on the abbreviations used in this table.} 
\label{results_hajjem_spatial}
\begingroup\footnotesize
\scalebox{0.9}{
\begin{tabular}{rllll|l}
  \hline
\hline
  & LaGaBst & LinearGP & LogitBst & mboost & LGBOOS \\ 
  \hline
Error & \bf{0.3085} & 0.3309 & 0.3501 & 0.3808 & 0.3068 \\ 
  (sd) & (0.0286) & (0.0278) & (0.0293) & (0.0336) & (0.027) \\ 
  \lbrack p-val\rbrack &  & [1.4e-22] & [6e-30] & [8.6e-41] & [0.057] \\ 
  NegLL & \bf{290.5} & 302.4 & 312.1 & 330.3 & 288.6 \\ 
  (sd) & (13.3) & (13.5) & (9.72) & (6.32) & (13.5) \\ 
  \lbrack p-val\rbrack &  & [8.8e-31] & [8.8e-40] & [6.2e-55] & [7.3e-16] \\ 
   \hline
Error\_ext & \bf{0.3755} & 0.3953 & 0.3986 & 0.4283 & 0.3732 \\ 
  (sd) & (0.0419) & (0.0306) & (0.055) & (0.0662) & (0.0396) \\ 
  \lbrack p-val\rbrack &  & [6.8e-07] & [3.3e-07] & [2.4e-14] & [0.044] \\ 
  NegLL\_ext & \bf{320} & 328.3 & 331.4 & 339.7 & 319.1 \\ 
  (sd) & (15.8) & (11.3) & (21) & (10.9) & (15.8) \\ 
  \lbrack p-val\rbrack &  & [8.9e-10] & [7.9e-11] & [7e-26] & [0.00027] \\ 
   \hline
RMSE $\sigma^2$ & 0.6237 & 0.4314 &  &  & 0.4944 \\ 
  RMSE $\rho$ & 0.1001 & 0.06225 &  &  & 0.05211 \\ 
  Bias $\sigma^2$ & -0.5945 & -0.3353 &  &  & -0.4398 \\ 
  Bias $\rho$ & 0.06441 & 0.02141 &  &  & 0.01426 \\ 
   \hline
Time (s) & 14.32 & 1.908 & 0.03377 & 0.7835 & 40.83 \\ 
   \hline
\hline
\end{tabular}
}
\endgroup
\end{table}

\else
\begin{table}[ht!]
\centering
\begingroup\footnotesize
\scalebox{0.9}{
\begin{tabular}{rllll|l}
  \hline
\hline
  & LaGaBoost & LinearME & LogitBoost & mboost & LaGaBoostOOS \\ 
  \hline
Error & \bf{0.2373} & 0.287 & 0.3393 & 0.2825 & 0.2373 \\ 
  (sd) & (0.00674) & (0.00885) & (0.00871) & (0.00958) & (0.00677) \\ 
  \lbrack p-val\rbrack &  & [1e-86] & [2.8e-106] & [6.4e-73] & [0.66] \\ 
  Log\_loss & \bf{2421} & 2785 & 3030 & 2917 & 2421 \\ 
  (sd) & (45.5) & (46.1) & (28.8) & (19.7) & (45.6) \\ 
  \lbrack p-val\rbrack &  & [3.1e-105] & [8.2e-114] & [6.5e-111] & [0.67] \\ 
   \hline
Error\_ext & \bf{0.3432} & 0.4206 & 0.3533 &  & 0.3434 \\ 
  (sd) & (0.00793) & (0.00898) & (0.015) &  & (0.00783) \\ 
  \lbrack p-val\rbrack &  & [1.9e-95] & [9.2e-12] &  & [0.31] \\ 
  Log\_loss\_ext & \bf{3028} & 3295 & 3078 &  & 3029 \\ 
  (sd) & (31.4) & (17.9) & (55) &  & (31.7) \\ 
  \lbrack p-val\rbrack &  & [1.2e-105] & [7.5e-18] &  & [0.0036] \\ 
   \hline
RMSE\_sigma2 & 0.2099 & 0.3589 &  &  & 0.2141 \\ 
  Bias\_sigma2 & -0.1953 & -0.3536 &  &  & -0.1941 \\ 
   \hline
Time (s) & 0.6646 & 0.03906 & 0.07785 & 16.96 & 3.276 \\ 
   \hline
\hline
\end{tabular}
}
\endgroup
\caption{Results for the grouped / high-cardinality categorical variable data and a binary Bernoulli likelihood. The results for the ``extrapolation" test data sets are denoted by `\_ext'.} 
\label{results_hajjem_one_way}
\end{table}

\begin{table}[ht!]
\centering
\begingroup\footnotesize
\scalebox{0.9}{
\begin{tabular}{rllll|l}
  \hline
\hline
  & LaGaBoost & LinearGP & LogitBoost & mboost & LaGaBoostOOS \\ 
  \hline
Error & \bf{0.3085} & 0.3309 & 0.3501 & 0.3808 & 0.3068 \\ 
  (sd) & (0.0286) & (0.0278) & (0.0293) & (0.0336) & (0.027) \\ 
  \lbrack p-val\rbrack &  & [1.4e-22] & [6e-30] & [8.6e-41] & [0.057] \\ 
  Log\_loss & \bf{290.5} & 302.4 & 312.1 & 330.3 & 288.6 \\ 
  (sd) & (13.3) & (13.5) & (9.72) & (6.32) & (13.5) \\ 
  \lbrack p-val\rbrack &  & [8.8e-31] & [8.8e-40] & [6.2e-55] & [7.3e-16] \\ 
   \hline
Error\_ext & \bf{0.3755} & 0.3953 & 0.3986 & 0.4283 & 0.3732 \\ 
  (sd) & (0.0419) & (0.0306) & (0.055) & (0.0662) & (0.0396) \\ 
  \lbrack p-val\rbrack &  & [6.8e-07] & [3.3e-07] & [2.4e-14] & [0.044] \\ 
  Log\_loss\_ext & \bf{320} & 328.3 & 331.4 & 339.7 & 319.1 \\ 
  (sd) & (15.8) & (11.3) & (21) & (10.9) & (15.8) \\ 
  \lbrack p-val\rbrack &  & [8.9e-10] & [7.9e-11] & [7e-26] & [0.00027] \\ 
   \hline
RMSE\_sigma2 & 0.6237 & 0.4314 &  &  & 0.4944 \\ 
  RMSE\_rho & 0.1001 & 0.06225 &  &  & 0.05211 \\ 
  Bias\_sigma2 & -0.5945 & -0.3353 &  &  & -0.4398 \\ 
  Bias\_rho & 0.06441 & 0.02141 &  &  & 0.01426 \\ 
   \hline
Time (s) & 14.32 & 1.908 & 0.03377 & 0.7835 & 40.83 \\ 
   \hline
\hline
\end{tabular}
}
\endgroup
\caption{Results for the spatial data and a binary Bernoulli likelihood.} 
\label{results_hajjem_spatial}
\end{table}

\fi

An alternative option to simulating additional training and test sets for choosing tuning parameters is to use cross-validation on the training data sets in every of the $100$ simulation runs. However, this is computationally more expensive as the number of simulation runs is relatively large. To investigate the differences between these two options for choosing tuning parameters, we redo the simulated experiments with $10$ simulation runs and choose tuning parameters using $4$-fold cross-validation on the training data in every simulation run. The results of this are reported in Tables \ref{results_hajjem_one_way_CV} and \ref{results_hajjem_spatial_CV} in the appendix. Overall, we observe only minor differences.

Next, we also perform the same simulated experiments using a Poisson likelihood with a logarithmic link function instead of a binary Bernoulli likelihood. Specifically, we simulate grouped and spatial random effects as described above with $\sigma^2=0.2$, and we simulate $F(X)$ according to \eqref{simF} with $C_1$ chosen as described above and $C_2$ chosen such that the variance of $F(X)$ is approximately $0.2$. Response variable data is then simulated from a Poisson distribution with mean equaling $\exp(F(X)+Zb)$. Tuning parameters are chosen similarly as for the binary data by minimizing the test negative Poisson likelihood. Further, we use the RMSE and the negative Poisson likelihood for evaluating prediction accuracy. The results of this are reported in Tables \ref{results_hajjem_one_way_poisson} and \ref{results_hajjem_spatial_poisson}. We find qualitatively very similar results as for the binary data. In particular, the LaGaBoost algorithm significantly outperforms all alternative approaches in all prediction accuracy measures for both the grouped and the spatial data.

\ifTwoColumn
\begin{table}[ht!]
\centering
\caption{Results for the grouped / high-cardinality categorical variable data and a Poisson likelihood. `PoisBst' denotes independent Newton boosting 
                                    with a Poisson likelihood. See the caption of Table 
                                    \ref{results_hajjem_one_way} for information 
                                    on the abbreviations used in this table.} 
\label{results_hajjem_one_way_poisson}
\begingroup\footnotesize
\scalebox{0.9}{
\begin{tabular}{rllll|l}
  \hline
\hline
  & LaGaBst & LinearME & PoisBst & mboost & LGBOOS \\ 
  \hline
RMSE & \bf{1.465} & 1.614 & 1.575 & 1.528 & 1.445 \\ 
  (sd) & (0.514) & (0.531) & (0.503) & (0.51) & (0.5) \\ 
  \lbrack p-val\rbrack &  & [2.5e-12] & [4e-55] & [9.6e-09] & [1.3e-05] \\ 
  NegLL & \bf{7062} & 7504 & 7515 & 7293 & 7033 \\ 
  (sd) & (215) & (285) & (235) & (239) & (179) \\ 
  \lbrack p-val\rbrack &  & [3.1e-62] & [6.8e-77] & [2.9e-39] & [8.8e-07] \\ 
   \hline
RMSE\_ext & \bf{1.516} & 1.599 & 1.536 &  & 1.504 \\ 
  (sd) & (0.482) & (0.481) & (0.481) &  & (0.468) \\ 
  \lbrack p-val\rbrack &  & [3.9e-75] & [1.1e-07] &  & [0.0066] \\ 
  NegLL\_ext & \bf{7488} & 7884 & 7546 &  & 7466 \\ 
  (sd) & (212) & (286) & (218) &  & (181) \\ 
  \lbrack p-val\rbrack &  & [1.2e-66] & [1.3e-14] &  & [7.9e-05] \\ 
   \hline
RMSE $\sigma^2$ & 0.05751 & 0.02506 &  &  & 0.04392 \\ 
  Bias $\sigma^2$ & -0.05382 & -0.00339 &  &  & -0.03661 \\ 
   \hline
Time (s) & 0.3154 & 0.02697 & 0.0606 & 8.992 & 5.368 \\ 
   \hline
\hline
\end{tabular}
}
\endgroup
\end{table}

\begin{table}[ht!]
\centering
\caption{Results for the spatial data and a Poisson likelihood. See the captions of Tables 
                                    \ref{results_hajjem_one_way} and \ref{results_hajjem_one_way_poisson} 
                                    for information on the abbreviations used in this table.} 
\label{results_hajjem_spatial_poisson}
\begingroup\footnotesize
\scalebox{0.9}{
\begin{tabular}{rllll|l}
  \hline
\hline
  & LaGaBst & LinearGP & PoisBst & mboost & LGBOOS \\ 
  \hline
RMSE & \bf{1.415} & 1.465 & 1.452 & 1.494 & 1.421 \\ 
  (sd) & (0.324) & (0.324) & (0.322) & (0.33) & (0.323) \\ 
  \lbrack p-val\rbrack &  & [4.7e-24] & [1.2e-17] & [4.6e-28] & [7.7e-05] \\ 
  NegLL & \bf{747.2} & 767 & 764.1 & 787.5 & 749.6 \\ 
  (sd) & (53.5) & (56.5) & (55) & (62.2) & (53) \\ 
  \lbrack p-val\rbrack &  & [2.4e-30] & [9.2e-22] & [1.9e-34] & [4.5e-05] \\ 
   \hline
RMSE\_ext & \bf{1.505} & 1.528 & 1.525 & 1.554 & 1.507 \\ 
  (sd) & (0.624) & (0.623) & (0.623) & (0.621) & (0.622) \\ 
  \lbrack p-val\rbrack &  & [2.7e-09] & [1e-08] & [2.1e-21] & [0.17] \\ 
  NegLL\_ext & \bf{773.4} & 785.6 & 783.1 & 799.1 & 774.2 \\ 
  (sd) & (94) & (96.1) & (95.8) & (92.1) & (91.1) \\ 
  \lbrack p-val\rbrack &  & [4.6e-09] & [2.4e-09] & [3.8e-24] & [0.36] \\ 
   \hline
RMSE $\sigma^2$ & 0.09264 & 0.1217 &  &  & 0.1256 \\ 
  RMSE $\rho$ & 0.06708 & 0.06526 &  &  & 0.06721 \\ 
  Bias $\sigma^2$ & -0.06304 & 0.08664 &  &  & 0.08688 \\ 
  Bias $\rho$ & -0.003803 & -0.05122 &  &  & -0.05566 \\ 
   \hline
Time (s) & 10.16 & 2.022 & 0.04036 & 0.689 & 26.77 \\ 
   \hline
\hline
\end{tabular}
}
\endgroup
\end{table}

\else
\begin{table}[ht!]
\centering
\begingroup\footnotesize
\scalebox{0.9}{
\begin{tabular}{rllll|l}
  \hline
\hline
  & LaGaBoost & LinearME & PoisonBoost & mboost & LaGaBoostOOS \\ 
  \hline
RMSE & \bf{1.465} & 1.614 & 1.575 & 1.528 & 1.445 \\ 
  (sd) & (0.514) & (0.531) & (0.503) & (0.51) & (0.5) \\ 
  \lbrack p-val\rbrack &  & [2.5e-12] & [4e-55] & [9.6e-09] & [1.3e-05] \\ 
  Neg\_log\_lik & \bf{7062} & 7504 & 7515 & 7293 & 7033 \\ 
  (sd) & (215) & (285) & (235) & (239) & (179) \\ 
  \lbrack p-val\rbrack &  & [3.1e-62] & [6.8e-77] & [2.9e-39] & [8.8e-07] \\ 
   \hline
RMSE\_ext & \bf{1.516} & 1.599 & 1.536 &  & 1.504 \\ 
  (sd) & (0.482) & (0.481) & (0.481) &  & (0.468) \\ 
  \lbrack p-val\rbrack &  & [3.9e-75] & [1.1e-07] &  & [0.0066] \\ 
  Neg\_log\_lik\_ext & \bf{7488} & 7884 & 7546 &  & 7466 \\ 
  (sd) & (212) & (286) & (218) &  & (181) \\ 
  \lbrack p-val\rbrack &  & [1.2e-66] & [1.3e-14] &  & [7.9e-05] \\ 
   \hline
RMSE\_sigma2 & 0.05751 & 0.02506 &  &  & 0.04392 \\ 
  Bias\_sigma2 & -0.05382 & -0.00339 &  &  & -0.03661 \\ 
   \hline
Time (s) & 0.3154 & 0.02697 & 0.0606 & 8.992 & 5.368 \\ 
   \hline
\hline
\end{tabular}
}
\endgroup
\caption{Results for the grouped / high-cardinality categorical variable data and a Poisson likelihood.} 
\label{results_hajjem_one_way_poisson}
\end{table}

\begin{table}[ht!]
\centering
\begingroup\footnotesize
\scalebox{0.9}{
\begin{tabular}{rllll|l}
  \hline
\hline
  & LaGaBoost & LinearGP & PoisonBoost & mboost & LaGaBoostOOS \\ 
  \hline
RMSE & \bf{1.415} & 1.465 & 1.452 & 1.494 & 1.421 \\ 
  (sd) & (0.324) & (0.324) & (0.322) & (0.33) & (0.323) \\ 
  \lbrack p-val\rbrack &  & [4.7e-24] & [1.2e-17] & [4.6e-28] & [7.7e-05] \\ 
  Neg\_log\_lik & \bf{747.2} & 767 & 764.1 & 787.5 & 749.6 \\ 
  (sd) & (53.5) & (56.5) & (55) & (62.2) & (53) \\ 
  \lbrack p-val\rbrack &  & [2.4e-30] & [9.2e-22] & [1.9e-34] & [4.5e-05] \\ 
   \hline
RMSE\_ext & \bf{1.505} & 1.528 & 1.525 & 1.554 & 1.507 \\ 
  (sd) & (0.624) & (0.623) & (0.623) & (0.621) & (0.622) \\ 
  \lbrack p-val\rbrack &  & [2.7e-09] & [1e-08] & [2.1e-21] & [0.17] \\ 
  Neg\_log\_lik\_ext & \bf{773.4} & 785.6 & 783.1 & 799.1 & 774.2 \\ 
  (sd) & (94) & (96.1) & (95.8) & (92.1) & (91.1) \\ 
  \lbrack p-val\rbrack &  & [4.6e-09] & [2.4e-09] & [3.8e-24] & [0.36] \\ 
   \hline
RMSE\_sigma2 & 0.09264 & 0.1217 &  &  & 0.1256 \\ 
  RMSE\_rho & 0.06708 & 0.06526 &  &  & 0.06721 \\ 
  Bias\_sigma2 & -0.06304 & 0.08664 &  &  & 0.08688 \\ 
  Bias\_rho & -0.003803 & -0.05122 &  &  & -0.05566 \\ 
   \hline
Time (s) & 10.16 & 2.022 & 0.04036 & 0.689 & 26.77 \\ 
   \hline
\hline
\end{tabular}
}
\endgroup
\caption{Results for the spatial data and a Poisson likelihood.} 
\label{results_hajjem_spatial_poisson}
\end{table}

\fi

\subsection{When Does the LaGaBoost Algorithm Outperform Independent Boosting?}\label{simul_ext}
It is relatively obvious that the LaGaBoost algorithm tends to outperform linear latent Gaussian models when there are non-linearities and interactions. It is less clear in which situations the LaGaBoost algorithm outperforms classical boosting algorithms which include categorical variables and/or spatial locations in the predictor variables $X$ for $F(\cdot)$ and, conditionally on this, assume independence among samples. As mentioned in Section \ref{priorsreg}, intuitively, we expect that the improvement in prediction accuracy of our novel approach over independent tree-boosting is the larger, the smaller the number of observations per category of a categorical variable is and the faster the covariance decays over space and/or time. To analyze this, we repeat the above simulated experiments for the binary data with varying numbers of samples per group and varying range parameters $\rho$. Specifically, for the grouped random effects with $n=5000$ samples, we consider the following number of samples per group: $10$, $20$, $50$, $100$, and $200$. For the spatial data, we consider the following range parameters $\rho$: $0.1$, $0.2$, $0.5$, and $1$. Apart from this, we use the same experimental setup as above for the Bernoulli likelihood. 

Figure \ref{simul_compare} reports the relative decrease in the test error of the LaGaBoost algorithm compared to the LogitBoost algorithm as well as the average test error of the two algorithms for the interpolation test data sets. These results confirm our hypothesis that the improvement in prediction accuracy is the larger, the smaller the number of observations per group is and the faster the covariance decays over space. In other words, the higher the complexity of the underlying true function relative to the sample size, the larger is the improvement obtained by the LaGaBoost algorithm. We conjecture that this is not just due to more accurate learning of the random effects themselves but also more efficient learning of the remaining part of the predictor function $F(\cdot)$. Figure \ref{simul_compare} also shows that, as expected, average test errors of both algorithms decrease when having fewer categories for the categorical grouping variables and when the correlation decays slower of space.

\begin{figure}[ht!]
	\begin{center}
		\ifTwoColumn
		\includegraphics[width=0.23\textwidth]{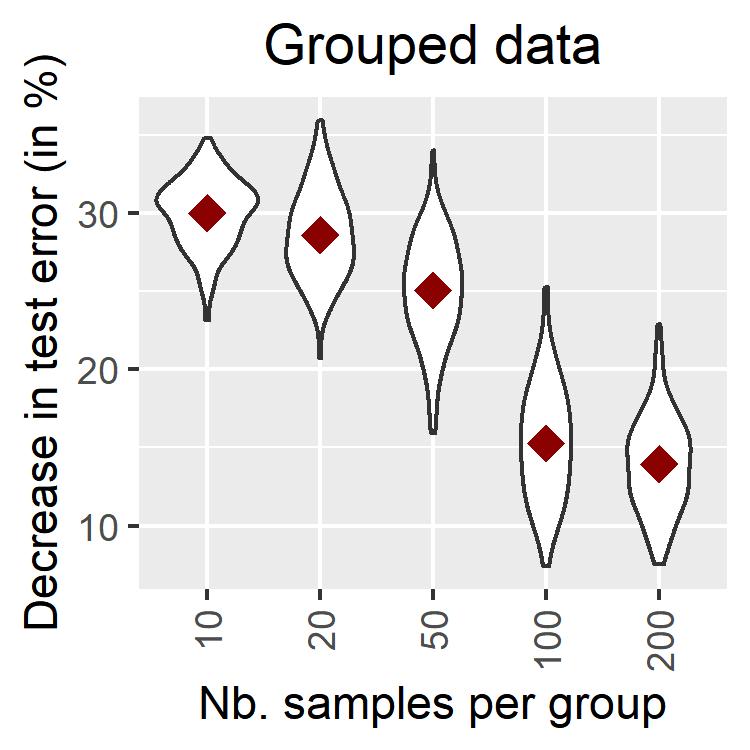}
		\includegraphics[width=0.23\textwidth]{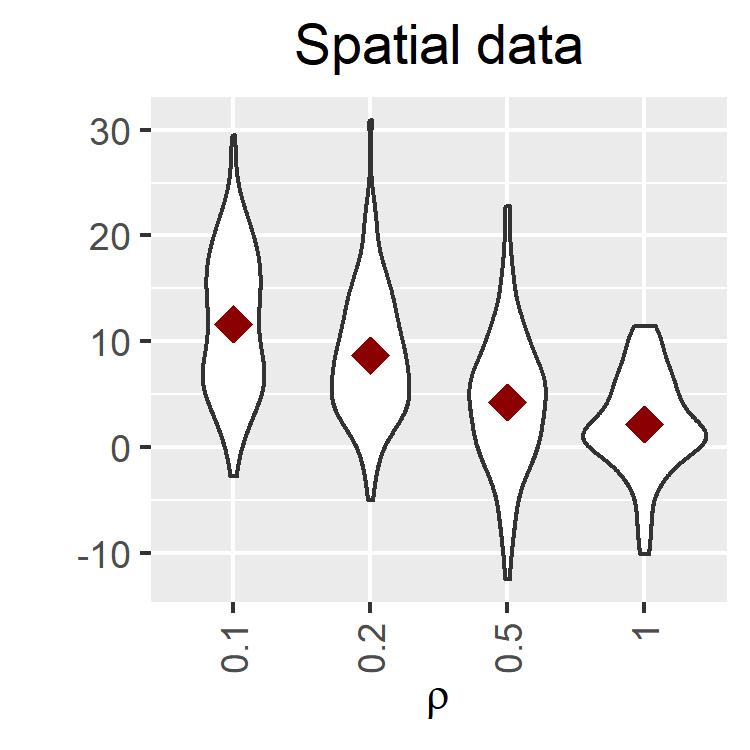}
		\else
		\includegraphics[width=0.4\textwidth]{simul_compare_one_way.jpeg}
		\includegraphics[width=0.4\textwidth]{simul_compare_spatial.jpeg}
		\fi
		\ifTwoColumn
		\includegraphics[width=0.46\textwidth]{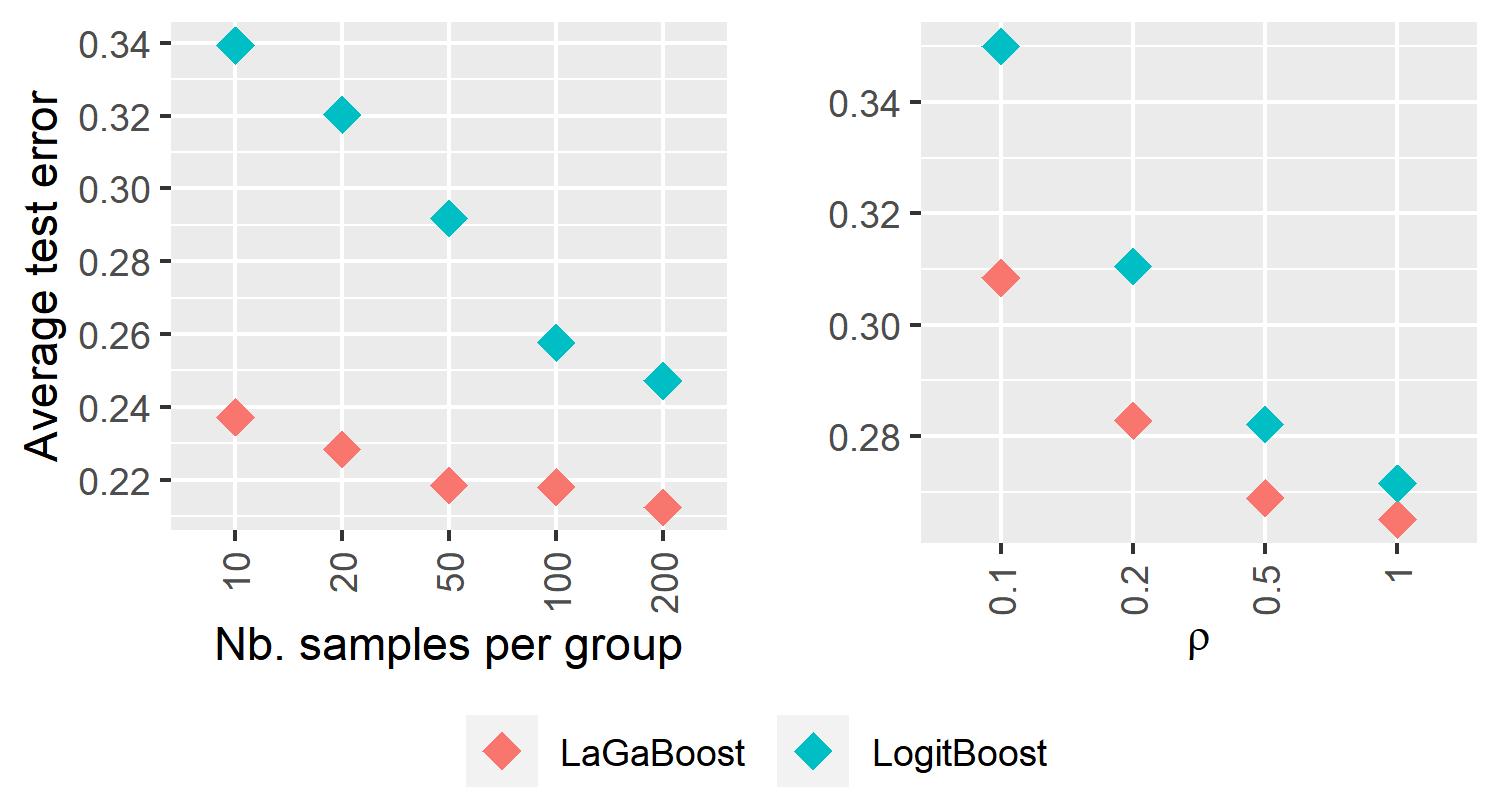}
		\else
		\includegraphics[width=0.8\textwidth]{simul_compare_raw.jpeg}
		\fi
		\caption{Comparison of the LaGaBoost and LogitBoost algorithms for grouped data with varying number of samples per group and spatial data with varying range parameters $\rho$. The top row shows the relative decrease in test error of the LaGaBoost vs. the LogitBoost algorithm visualized using violin plots. The red rhombi representing means over the simulation runs. The bottom row shows the average test error of the two algorithms.}
		\label{simul_compare} 
	\end{center}
\end{figure}

\section{Real-world Applications}\label{data_appl}
In the following, we apply the LaGaBoost algorithm to two real-world binary classification data sets and compare its prediction accuracy to alternative approaches. We consider both grouped data with a high-cardinality categorical variable and a spatial data set. For the former, we consider data on poverty among young females in the US collected by the National Longitudinal Study of Youth (NLSY) and available from \url{https://www3.nd.edu/~rwilliam/statafiles/teenpovxt.dta}. Here, the person id number is the high-cardinality categorical variable that determines the grouped random effects, and the goal is to predict the binary poverty indicator. As a spatial data set, we consider species distribution data. Specifically, we use presence-absence data on rainforest understorey vascular plants in North-east New South Wales, Australia, species ``nsw43" obtained from the \texttt{disdat} R package \citep{elith2020presence}. The goal is to predict the presence or absence of the species. Table \ref{data_sum} summarizes the data sets.

\ifTwoColumn
\begin{table}
	\centering
	\caption{Summary of real-world data sets. `\# data' denotes the sample size, `\# features' the number of predictor variables, and `\# cat.' the number of groups of the high-cardinality categorical variable.} 
	\label{data_sum}
	\begin{tabular}{llrrrr}
		\hline
		\hline
		Name & Data type & \# data & Freq. of $1$'s& \# features & \# cat. \\ 
		\hline
		Poverty & Grouped & 5755 & 36.79 \% & 7 + 1 & 1151\\ 
		Species &  Spatial & 909 & 23.76 \% &  13 + 2 &  \\ 
		\hline
		\hline
	\end{tabular}
\end{table}
\else
\begin{table}[ht!]
	\centering
	\begin{tabular}{llrrrr}
		\hline
		\hline
		Name & Data type & Nb. samples & Freq. of $1$'s& Nb. features & Nb. categories \\ 
		\hline
		Poverty & Grouped  & 5755 & 36.79 \% & 7 + 1 & 1151\\ 
		Species &  Spatial & 909 & 23.76 \% &  13 + 2 &  \\ 
		\hline
		\hline
	\end{tabular}
	\caption{Summary of real-world data sets.} 
\label{data_sum}
\end{table}
\fi

We compare the LaGaBoost algorithm to the same alternative approaches as in the simulated experiments in Section \ref{simul} using nested $4$-fold cross-validation. For the grouped poverty data, we perform stratified cross-validation such that every fold contains approximately the same amount of data for every category of the grouping variable. A reason for doing this is that one of the alternative approaches, the \texttt{mboost} R package, does not allow for making predictions for unobserved groups. Tuning parameters are chosen by doing an additional inner $4$-fold cross-validation on every of the four training data sets.\footnote{Note that one has to be careful when doing cross-validation for dependent data to avoid biased estimates of the generalization error as pointed out by, e.g., \citet{rabinowicz2020cross}. However, apart from the fact that the validation and test data sets are of slightly different sizes, our cross-validation setting preserves the distributional relation between the inner fold training and validation data sets and the training and test data sets and, consequently, no bias is introduced.} We consider the same set of tuning parameters and selection criterion as in the simulated experiments

The results are reported in Table \ref{results_real_short}. In addition to the test error and the test log loss, we also report the test area under the ROC curve (AUC). We find that the LaGaBoost algorithm outperforms all alternative methods in all three prediction accuracy metrics for both the grouped data with a high-cardinality categorical variable and the spatial data.

\ifTwoColumn
\begin{table}[ht!]
\centering
\caption{Results for the real-world data sets. 
  `Linear' denotes the linear grouped mixed effects and linear Gaussian process models.} 
\label{results_real_short}
\begingroup\footnotesize
\scalebox{1}{
\begin{tabular}{rllll}
  \hline
\hline
  & LaGaBoost & Linear & LogitBoost & mboost \\ 
  \hline\multicolumn{5}{l}{Poverty (grouped / high-cardinality categorical data)}\\\hline Error & \bf{0.2792} & 0.2848 & 0.3239 & 0.3197 \\ 
  AUC & \bf{0.7318} & 0.7313 & 0.7001 & 0.7028 \\ 
  Log\_loss & \bf{0.5789} & 0.5878 & 0.6013 & 0.6081 \\ 
   \hline\multicolumn{5}{l}{Species (spatial data)}\\\hline Error & \bf{0.2365} & 0.3003 & 0.2717 & 0.2728 \\ 
  AUC & \bf{0.7383} & 0.7061 & 0.683 & 0.635 \\ 
  Log\_loss & \bf{0.4933} & 0.5625 & 0.527 & 0.5662 \\ 
   \hline
\hline
\end{tabular}
}
\endgroup
\end{table}

\else
\begin{table}[ht!]
\centering
\begingroup\footnotesize
\scalebox{0.9}{
\begin{tabular}{rllll}
  \hline
\hline
  & LaGaBoost & Linear & LogitBoost & mboost \\ 
  \hline\multicolumn{5}{l}{Poverty (grouped / high-cardinality categorical data)}\\\hline Error & \bf{0.2792} & 0.2848 & 0.3239 & 0.3197 \\ 
  AUC & \bf{0.7318} & 0.7313 & 0.7001 & 0.7028 \\ 
  Log\_loss & \bf{0.5789} & 0.5878 & 0.6013 & 0.6081 \\ 
   \hline\multicolumn{5}{l}{Species (spatial data)}\\\hline Error & \bf{0.2365} & 0.3003 & 0.2717 & 0.2728 \\ 
  AUC & \bf{0.7383} & 0.7061 & 0.683 & 0.635 \\ 
  Log\_loss & \bf{0.4933} & 0.5625 & 0.527 & 0.5662 \\ 
   \hline
\hline
\end{tabular}
}
\endgroup
\caption{Results for the real-world data sets. 
  `Linear' denotes the linear grouped mixed effects and linear Gaussian process models.} 
\label{results_real_short}
\end{table}

\fi

\section{Conclusion}
We have introduced a novel way for combining latent Gaussian models, such as Gaussian processes and random effects models, with boosting. This is done by applying functional gradient descent to the negative logarithmic marginal likelihood of a generalized mixed effects model in a boosting framework while jointly learning hyperparameters. We have obtained increased prediction accuracy compared to existing approaches in both simulated and real-world data experiments. Future research can investigate how the approximation used for the marginal likelihood impacts properties such as prediction accuracy and computational time of the LaGaBoost algorithm.

\ifTwoColumn
\section*{Acknowledgments}
We are thankful to the anonymous reviewers for their valuable comments which helped to improve the paper. This research was partially supported by the Swiss Innovation Agency - Innosuisse (grant number `55463.1 IP-IC').
\fi

\bibliography{bib_LaGaBoost}

\clearpage

\appendix

\section{Appendix}
\setcounter{table}{0}
\counterwithin{table}{section}

\subsection*{Additional Results for the Simulated Experiments}\label{add_res}

\ifTwoColumn
\begin{table}[ht!]
\centering
\caption{Results for the grouped data and a binary Bernoulli likelihood based on only 10 simulation runs 
                                and using 4-fold cross-validation on the training data for determining tuning parameters.} 
\label{results_hajjem_one_way_CV}
\begingroup\footnotesize
\scalebox{0.9}{
\begin{tabular}{rllll}
  \hline
\hline
  & LaGaBst & LinearME & LogitBst & mboost \\ 
  \hline
Error & \bf{0.2371} & 0.2878 & 0.3365 & 0.288 \\ 
  (sd) & (0.00366) & (0.0079) & (0.00859) & (0.0109) \\ 
  NegLL & \bf{2429} & 2785 & 3007 & 2921 \\ 
  (sd) & (26.9) & (39.7) & (33.9) & (21) \\ 
   \hline
Error\_ext & \bf{0.3429} & 0.4158 & 0.3656 &  \\ 
  (sd) & (0.00796) & (0.00991) & (0.0269) &  \\ 
  NegLL\_ext & \bf{3021} & 3293 & 3171 &  \\ 
  (sd) & (34.8) & (18.6) & (223) &  \\ 
   \hline
\hline
\end{tabular}
}
\endgroup
\end{table}

\begin{table}[ht!]
\centering
\caption{Results for the spatial data and a binary Bernoulli likelihood based on only 10 simulation runs 
                                and using 4-fold cross-validation on the training data for determining tuning parameters.} 
\label{results_hajjem_spatial_CV}
\begingroup\footnotesize
\scalebox{0.9}{
\begin{tabular}{rllll}
  \hline
\hline
  & LaGaBst & LinearGP & LogitBst & mboost \\ 
  \hline
Error & \bf{0.3114} & 0.3304 & 0.3534 & 0.3538 \\ 
  (sd) & (0.0321) & (0.0279) & (0.0339) & (0.026) \\ 
  NegLL & \bf{291.2} & 300.4 & 313.2 & 313.6 \\ 
  (sd) & (15.9) & (14.9) & (12.3) & (14.8) \\ 
   \hline
Error\_ext & \bf{0.389} & 0.4074 & 0.4016 & 0.439 \\ 
  (sd) & (0.0297) & (0.0396) & (0.0484) & (0.0411) \\ 
  NegLL\_ext & \bf{323.7} & 330.2 & 336.8 & 378.6 \\ 
  (sd) & (17.6) & (14.3) & (18) & (66.4) \\ 
   \hline
\hline
\end{tabular}
}
\endgroup
\end{table}

\else
\begin{table}[ht!]
\centering
\begingroup\footnotesize
\scalebox{0.9}{
\begin{tabular}{rllll}
  \hline
\hline
  & LaGaBoost & LinearME & LogitBoost & mboost \\ 
  \hline
Error & \bf{0.2371} & 0.2878 & 0.3365 & 0.288 \\ 
  (sd) & (0.00366) & (0.0079) & (0.00859) & (0.0109) \\ 
  Log\_loss & \bf{2429} & 2785 & 3007 & 2921 \\ 
  (sd) & (26.9) & (39.7) & (33.9) & (21) \\ 
   \hline
Error\_ext & \bf{0.3429} & 0.4158 & 0.3656 &  \\ 
  (sd) & (0.00796) & (0.00991) & (0.0269) &  \\ 
  Log\_loss\_ext & \bf{3021} & 3293 & 3171 &  \\ 
  (sd) & (34.8) & (18.6) & (223) &  \\ 
   \hline
\hline
\end{tabular}
}
\endgroup
\caption{Results for the grouped data and a binary Bernoulli likelihood based on only 10 simulation runs 
                                and using 4-fold cross-validation on the training data for determining tuning parameters.} 
\label{results_hajjem_one_way_CV}
\end{table}

\begin{table}[ht!]
\centering
\begingroup\footnotesize
\scalebox{0.9}{
\begin{tabular}{rllll}
  \hline
\hline
  & LaGaBoost & LinearGP & LogitBoost & mboost \\ 
  \hline
Error & \bf{0.3114} & 0.3304 & 0.3534 & 0.3538 \\ 
  (sd) & (0.0321) & (0.0279) & (0.0339) & (0.026) \\ 
  Log\_loss & \bf{291.2} & 300.4 & 313.2 & 313.6 \\ 
  (sd) & (15.9) & (14.9) & (12.3) & (14.8) \\ 
   \hline
Error\_ext & \bf{0.389} & 0.4074 & 0.4016 & 0.439 \\ 
  (sd) & (0.0297) & (0.0396) & (0.0484) & (0.0411) \\ 
  Log\_loss\_ext & \bf{323.7} & 330.2 & 336.8 & 378.6 \\ 
  (sd) & (17.6) & (14.3) & (18) & (66.4) \\ 
   \hline
\hline
\end{tabular}
}
\endgroup
\caption{Results for the spatial data and a binary Bernoulli likelihood based on only 10 simulation runs 
                                and using 4-fold cross-validation on the training data for determining tuning parameters.} 
\label{results_hajjem_spatial_CV}
\end{table}

\fi

\end{document}